\documentclass{article}

\usepackage[utf8]{inputenc} 
\usepackage[T1]{fontenc}    
\usepackage{hyperref}       
\usepackage{url}            
\usepackage{booktabs}       
\usepackage{amsfonts}       
\usepackage{nicefrac}       
\usepackage{microtype}      
\usepackage{xcolor}         

\usepackage{amsthm}
\usepackage{amsmath} 
\usepackage[pdftex]{graphicx}
\usepackage{algorithm}
\usepackage{svg}
\usepackage{algorithmicx}
\usepackage{comment}
\usepackage{wrapfig}
\usepackage{caption}
\usepackage{subcaption}
\usepackage{mathtools}
\usepackage{authblk}
\usepackage{algorithm}
\usepackage{algpseudocode}

\algnewcommand{\Input}{\State \textbf{Input: }}

\usepackage[numbers, compress]{natbib}

\newtheorem{theorem}{Theorem}[section]
\newtheorem{corollary}{Corollary}[theorem]

\newtheorem{proposition}[theorem]{Proposition}
\newtheorem{definition}[theorem]{Definition}
\newtheorem{example}[theorem]{Example}

\DeclareMathOperator*{\argmin}{\arg\min}

\makeatletter
\newcommand{\markthis}[3]{
  \overset{
    \textup{\makebox[0pt]{#1}}%
    \def\@currentlabel{#1}%
    \ltx@label{#2}%
  }{
    #3%
  }%
}
\makeatother

\definecolor{optimization_step}{RGB}{141, 160, 203}
\definecolor{deployment_step}{RGB}{229, 196, 148}

\title{Strategically Deceptive Model Deployment in Performative Prediction%
\thanks{Accepted at the ACM Conference on Fairness, Accountability, and Transparency (FAccT) 2026.}
}
\date{\vspace{-5ex}}  

\author[1]{Javier Sanguino Bautiste}
\author[1]{Thomas Kehrenberg}
\author[1,2]{Jose A.\@ Lozano}
\author[1,3]{Novi Quadrianto}
\affil[1]{Basque Center for Applied Mathematics, Bilbao, Spain}
\affil[2]{University of the Basque Country UPV/EHU, Spain}
\affil[3]{University of Sussex, Brighton, UK}
\affil[ ]{\texttt {\{jsanguino,tkehrenberg,jlozano,nquadrianto\}@bcamath.org}}

\begin{document}

\maketitle

\begin{abstract}
Machine Learning systems are increasingly deployed in decision-making settings that shape user behavior and, in turn, the data on which future decisions are based. Performative Prediction (PP) formalizes this feedback loop by modeling how deployed models induce distributional shifts. It studies how to learn robust and well-performing models under such dynamics. However, existing PP frameworks typically assume that the model governing these decisions is the same model observed by users (therefore, to which they respond). In practice, deployer institutions may instead disclose curated models, while internally relying on distinct opaque models.

We introduce Decoupled Performative Prediction (DPP), a framework that explicitly models mismatches between the model governing institutional decisions and the model that shapes user behavior. By analyzing the resulting optimization landscape, we show that DPP admits new different solutions that provably achieve lower risk for the institution than those under classical PP. We further propose an algorithm with provable convergence guarantees under standard assumptions, demonstrating how easy institutions can benefit from strategically deceptive deployment when they control model disclosure and users lack countervailing power. To capture the implications of such behavior, we introduce the deception cost, a quantitative measure of the degree of deception experienced by users. We study settings in which institutions incorporate this cost into the optimization process, motivated by reputational concerns or potential user abandonment, and show that such self-imposed constraints are insufficient to protect users. Overall, our results demonstrate that model disclosure is not merely an ethical consideration but a core technical design decision, underscoring the need for regulations that hold institutions accountable for deceptive deployment practices.
\end{abstract}

\section{Introduction}
Over the past decade, Machine Learning (ML) systems have become central instruments of institutional decision-making, shaping access to credit, housing, employment, education, and public services. 
For example, a bank may publish a model that estimates whether an applicant will be granted a loan based on financial attributes. 
Prospective borrowers, aware of these criteria, can strategically modify their financial behaviour to increase their perceived chances of approval.
When individuals adapt their actions in response to algorithmic criteria, the model does not merely predict outcomes; it actively alters the data on which future decisions are based.

%


Performative Prediction (PP) provides a formal framework for studying these post-deployment effects in learning systems \cite{peformativeprediction2020perdomo}. It captures the dynamic interaction between deployed predictors and data by explicitly modeling how deployment induces systematic, decision-dependent shifts in the data-generating process. More concretely, PP introduces a distribution map, which maps each deployed model with parameters $\theta$
to the data distribution $\mathcal{D}(\theta)$ induced after deployment.
%
%
However, PP implicitly assumes that the same model both governs institutional decisions and is observed by individuals, who then adapt their behavior accordingly. In real-world deployments, this rarely holds.
Consider a setting in which a bank discloses a simplified credit scoring model to applicants while relying internally on a more complex system to determine actual loan terms. 
As a result, individuals adapt their behavior to an observed model that differs from the one that actually governs outcomes.

This decoupling fundamentally alters the closed-loop interaction between model and data: institutional decisions are made using one model, while data is influenced by another. Individuals may strategically change their financial behavior, employment histories, or housing choices in response to disclosed criteria, even though their outcomes are ultimately determined by a different, opaque decision rule. These mismatches can expose individuals to hidden risks that are invisible under standard PP abstractions.

Crucially, these dynamics are structured by pronounced power asymmetries. Institutions unilaterally control model design, objectives, and deployment, while users are numerous, uncoordinated, and subject to algorithmic decisions without effective avenues for contestation. In other words, institutions can exercise what has been termed \textit{performative power} \cite{performativepower2022hard}: the capacity to shape individual behavior and, in turn, the data-generating environment in ways that support their objectives, while users have little ability to exert reciprocal influence. Moreover, when institutions strategically misrepresent deployed models, they can further steer user behavior while retaining control over the true decision criteria.



\subsection{Contributions}

In this paper, we introduce Decoupled Performative Prediction (DPP), a framework that explicitly models mismatches between used and perceived models. DPP formalizes settings in which institutional decisions are governed by a model $\theta_M$, while user behavior is shaped by a distinct disclosed model $\theta_D$, inducing a distribution map $\mathcal{D}(\theta_D)$.
DPP enables the formal analysis of institution–user interactions under strategic disclosure.
%
Specifically, our contributions are:

\begin{itemize}
    \item We formalize mismatches between deployed and perceived models within the DPP framework, enabling the systematic study of opaque and strategically deceptive deployment.
    \item We define the decoupled optimum, an interest point within the framework, show that it outperforms standard PP interest points, and propose an algorithm that provably converges to it under typical assumptions.
    \item We introduce a quantitative measure to know how costly to the users is the deceptive behavior of the institution, the \textit{deception cost}, providing a principled tool for assessing deployment risks in high-stakes settings.  
\end{itemize}

Our results show that performative feedback is not merely a function of deployed models, but of what users are led to believe those models are. 
This highlights disclosure as a central design variable in performative systems and reframes transparency not just as an ethical concern, but as a core technical factor in learning under feedback.

The remainder of the paper is organized as follows: Section~\ref{sec:related-work} reviews related work on performative prediction and performative power, Section~\ref{sec:performative-prediction} introduces the core concepts of PP, Section~\ref{sec:decoupled-performative-prediction} presents the proposed DPP framework, Section~\ref{sec:deceptive-cost} defines the deception cost and Section~\ref{sec:experiments} reports and discusses our experimental results.

\section{Related work}
\label{sec:related-work}

Performativity has its origin in performative utterances in language, which are sentences that change reality simply by being stated. They are expressions that do not merely describe or report, but rather constitute an action in themselves, such that uttering the sentence is part of performing the action it names \cite{austin1975things}. 
Some examples are “I declare you husband and wife,” which performs the act of marrying, or “I am sorry,” which performs the act of apologizing. It has since been applied across various fields such as journalism \cite{broersma2010journalism} and economics \cite{mackenzie2020economists} to explain scenarios where articulations do not merely reflect reality but actively shape it. 
In the context of ML, it was first introduced in \citet{peformativeprediction2020perdomo} as PP to study an obvious but overlooked problem: the interaction between model and world upon deployment, which was previously conceptualized as a feedback loop \cite{liu2018delayed,ensign2018runaway}. PP seeks to find robust and good-performing models when the model not only generates predictions on the environment but also shifts the distribution from which data is drawn, thereby shaping the data-generating process.

In their seminal work, \citet{peformativeprediction2020perdomo} established two solution concepts: \emph{performative stability}, which gives a sense of robustness to the distribution shift, and \emph{performative optimality}, which is the best performing model in this closed-loop interaction between model and data. They presented two algorithms, based on re-training, for reaching performative stability under some strong conditions: Repeated Gradient Descent (RGD) and Repeated Risk Minimization (RRM). Under even stronger assumptions, the stable point is close to the optimal point.
Subsequent work has studied convergence to the stable point in performative settings under increasingly general assumptions, including stateful environments \cite{brown2022performative}, extensions to reinforcement learning with policy-dependent environments and regularization-based convergence to stable policies \cite{mandal2023performative}, and relaxed assumptions that enable the study of convergence in neural networks  \cite{mofakhami2023performative}. Stochastic versions of the algorithms are also exist \cite{mendler2020stochastic, drusvyatskiy2023stochastic,li2024stochastic}.  
However, \citet{kabra2024limitations} showed that retraining can converge to performatively stable points that are far from the performatively optimal solution, or fail to converge altogether under simple and practically relevant performative shifts, particularly in finite-sample regimes, highlighting limitations of existing guarantees for retraining algorithms in real-world deployment. 

To optimize directly for the more interesting optimal point, one needs information about the distribution map. 
The challenge is then about calculating the associated \textit{performative gradient}. \citet{izzo2021learn} addressed this by using the REINFORCE method \cite{williams1992simple}. The method, however, is limited to distribution maps with a known functional form. \citet{miller2021outside} generalised it by assuming a linear scale family of the distribution map.
\citet{cyffers2024optimal} proposed using the reparametrization trick \cite{kingma2013auto} to calculate it by assuming a push-forward model for the distribution map.

Although our work also changes the original assumptions of PP and studies convergence properties, it is an extension of the classical PP framework rather than only a modification. Our framework strictly generalizes the classical setting by allowing the model to which users respond to differ from the model actually deployed by the institution. 
By formalizing decoupled deployment and disclosure, our work introduces a new modeling layer for studying institutional power, strategic misrepresentation, and algorithmic deception within performative systems. 
This reframes disclosure not merely as a transparency or explainability concern, but as a core technical design variable that governs behavioral feedback, welfare outcomes, and institutional advantage.

Our work is not the first to consider settings in which users do not fully or directly react to deployed models. Prior work has examined learning in performative settings with approximated distribution maps \citep{lin2024plugin,xue2024distributionally,wang2026wasserstein}, as well as learning in the neighboring literature on strategic classification \citep{HarMegPapWoo16}, where users respond based on estimates of the deployed model rather than full knowledge of it \citep{ghalme2021strategic}. The DPP framework we introduce is sufficiently general to accommodate these settings. However, in this work, we focus on the comparatively understudied setting in which model disclosure is not merely a source of incomplete information but an endogenous design choice that shapes performative feedback, learning dynamics, and institutional incentives over time. While we restrict attention to this setting in this paper, we emphasize that the framework naturally extends beyond it.

Also, closely related to our work is the study of how power exercised by different parties can influence and modify learning systems.

On the one hand, \textit{collective action} has explored how groups of agents coordinating their behavior can influence learning systems \citep{HarMazMenZrn23,GauBacJor25,KarVinKarSun25}. 
While collective action frameworks focus on user-side strategic coordination against a given model, our work instead focuses on institution-side strategic control of information that shapes individual responses in the first place. 
These perspectives are complementary: studying institutional influence through misaligned disclosure could hint at how users could organise responses to asymmetric power dynamics.

On the other hand, the role that institutions can have in modifying the environment through model deployment has been conceptualized as \emph{performative power} \cite{performativepower2022hard}, which quantifies an institution’s ability to steer users’ behavior—and thereby the induced data distribution—in directions that benefit the institution. 
Rather than measuring an institution’s ability to directly reshape the distribution map itself, we assume that institutions exert performative power by strategically disclosing, approximating, or obscuring deployed decision systems. In this setting, institutions intentionally induce systematic mismatches between the model that shapes user behavior and the model that ultimately determines outcomes, thereby exercising performative power while engaging in misleading behavior without being subject to corresponding accountability.
%
Accordingly, our analysis centers on the consequences of informational asymmetry and strategic model disclosure. We do not require an explicit notion of performative power in our formalization; we instead introduce a model of \emph{institutional deception} that captures a distinct and underexplored dimension of institutional control over performative systems.

\section{Background on Performative Prediction}
\label{sec:performative-prediction}

In this section, we introduce the basic notions of PP, including its key objectives and main algorithms. Apart from introducing enough background to understand the rest of the paper, the aim of this section is to facilitate a direct comparison with Section \ref{sec:decoupled-performative-prediction}, where we transfer these concepts to our framework, \textit{Decoupled} PP. 

Let $\mathcal{Z}$ denote the data space of interest.
In the case of classification, we would have \(\mathcal{Z}=\mathcal{X}\times \mathcal{Y}\) where \(\mathcal{X}\) is the feature space and \(\mathcal{Y}\) is the label space.
The goal is to learn a model 
with model parameters \(\theta\in\Theta\), and the learning objective is to minimize a loss function \(\ell:\mathcal{Z}\times \Theta\to \mathbb{R}\).

The main characteristic of PP problems is, however, that there is no fixed data distribution over which to minimize the expected loss.
Once deployed, a model \(\theta\) induces a shift in the distribution. 
This dependency is modeled by the \emph{distribution map} \(\mathcal{D}:\Theta\to \Delta(\mathcal{Z})\),
where \(\Delta(\mathcal{Z})\) denotes the set of distributions over \(\mathcal{Z}\).
One can evaluate the model's performance under this distribution shift with the \emph{performative risk}
%
%
\begin{equation}
    \mathcal{PR}(\theta) = \mathbb{E}_{z \sim \mathcal{D}(\theta)} \big[ \ell(z; \theta) \big] ~,
    \label{eq:performative-risk}
\end{equation}
which is the risk on the distribution that the model induces, \emph{not} the distribution the model was trained on.

\subsection{Points of interest in PP: Stable and Optimal Points}

Two solutions have been defined for PP: a stable point ($\theta_{ST}$), which is a robust solution to the distribution shift, and an optimal point ($\theta_{OP}$), which is the minimum of the performative risk.

%
\begin{definition}
\label{def:stable_point}
A stable point $\theta_{ST}$ is a value that minimizes the loss for the distribution it induces:
\begin{equation}
\theta_{ST} = \argmin_{\theta \in \Theta} \mathbb{E}_{z \sim \mathcal{D}(\theta_{ST})}  \big[\ell(z, \theta)\big].
\end{equation}
\end{definition}

\begin{definition}
\label{def:optimal point}
An optimal point $\theta_{OP}$ is the value of minimal loss of the performative risk:
\begin{equation}
\theta_{OP} := \argmin_{\theta \in \Theta} \mathcal{PR}(\theta) = \argmin_{\theta \in \Theta} \mathbb{E}_{z \sim \mathcal{D}(\theta)}  \big[\ell(z, \theta)\big].
\end{equation}
\end{definition}

As the optimal point is the minimizer of the performative risk, it follows that $\mathcal{PR}(\theta_{OP}) \le \mathcal{PR}(\theta_{ST})$.

Let us introduce a minimal example to illustrate these concepts. We will build on top of this example throughout the paper to show the differences between PP and DPP. 

\begin{example}[Re-storied based on biased coin flip example of \citet{peformativeprediction2020perdomo}]
\label{ex:biased-coin}
Suppose the environment consists of applicants whose repayment outcome is binary: an applicant either repays their loan (outcome $1$) or does not repay their loan (outcome $0$). 
Because of performativity, the outcome depends not only on the applicant's feature $X\in\{-1,+1 \}$ but also on the institution's deployed model parameter $\theta\in{[0,1]}$. We model the true underlying outcome as $Y \mid X \sim \mathrm{Bernoulli}\!\left(\tfrac{1}{2} + \mu X + \varepsilon \theta X\right)$,
where $\mu \in \left(0, \tfrac{1}{2}\right)$ captures a baseline bias in repayment propensity, and $\varepsilon > 0$ quantifies the strength of the performative effect. We further assume $\varepsilon < \tfrac{1}{2} - \mu$ to ensure the success probability remains in $[0,1]$.
The term $\varepsilon\theta X$ captures how the deployed model influences applicant behaviour, thereby shifting the data-generating process.
The induced data distribution is given implicitly by the Bernoulli model above: $(X,Y) \sim \mathcal{D}(\theta)$, since $\mathbb{P}(X,Y) = \mathbb{P}(Y|X)\cdot\mathbb{P}(X)$.
The institution trains a linear regression model $f_{\theta}(x)=\frac{1}{2} + \theta x$ to predict the binary outcome $Y$, using a squared loss $\ell(x,y;\theta) = (y-f_\theta(x))^2$. 

Then, the stable point and the optimal point are 

$$\theta_{ST}=\frac{\mu}{1-\varepsilon}~,\quad \theta_{OP}=\frac{\mu}{1-2\varepsilon}~;$$ 

and the achieved performative risk values are 

$$\mathcal{PR}(\theta_\mathit{ST})=\frac{1}{4} - \frac{\mu^2}{(1-\varepsilon)^2} ~, \quad \mathcal{PR}(\theta_\mathit{OP})= \frac{1}{4} - \frac{\mu^2}{(1-2\varepsilon)}~.$$

Note that for the solutions to exist in the given domain of $\theta\in[0,1]$, we further require $0 \le \frac{\mu}{1-2\varepsilon} \le 1$. The details of these derivations can be consulted in Appendix~\ref{appendix:details-example}.
\end{example}

\subsection{Convergence in PP}

The optimization algorithms\footnote{As in the broader ML field, the algorithms are based on Gradient Descent (GD)} that converge to the stable point are based on the naive approach of retraining the model after the distribution shift. These algorithms, RRM (if fully retrained) or RGD (if retrained with only one gradient step), were proposed in the seminal work by \citet{peformativeprediction2020perdomo}. 
One only needs to wait for the post-deployment distribution shift, collect the shifted data, and retrain the model. Consequently, these algorithms train solely on the data distribution induced by the current parameters. They do not account for the underlying mechanism driving the shift (i.e., the distribution map), since they observe shifted samples but not how the distribution would evolve after redeployment. As a result, they can only converge to the stable point.
 \cite{peformativeprediction2020perdomo}.

%
%
%
The algorithm that converges to the optimal point, PerfGD \cite{izzo2021learn,cyffers2024optimal}, takes into account information of the distribution map. It is based on applying gradient descent directly to the performative risk, which leads to the following parameter update rule

$$\theta^{t+1} \leftarrow \theta^{t} - \mu ~ \nabla_{\theta}\mathcal{PR}(\theta).$$

The key step is to calculate the \textit{performative gradient} $\nabla_{\theta}\mathcal{PR}(\theta)$, which can be calculated with (a) the REINFORCE trick, if one has access to the probability density distribution of the distribution map, or (b) the reparametrization trick, if one assumes the distribution map can be written with a push-forward model (Definition~\ref{def:push-forward}). The mathematical details of both variants can be consulted in Appendix~\ref{appendix:perfgd-variants}.
The convergence of this algorithm to the optimal point has been studied by analyzing the convexity of the performative risk $\mathcal{PR}(\theta)$ \cite{miller2021outside}.

\section{Decoupled Performative Prediction (DPP)}
\label{sec:decoupled-performative-prediction}

There are many concrete settings in which institutions or companies attempt to influence or mislead individuals for their own benefit outside of ML. Examples include deceptive advertising claims, the use of dark patterns in user interfaces to steer users toward more expensive options, or the selective disclosure of information in pricing, contracts, and promotions. Such practices are widely recognized as harmful and are therefore subject to consumer-protection and advertising regulations in many jurisdictions.
Our main claim in this paper is that analogous forms of manipulation can also occur through the deployment of predictive models. This risk is exacerbated by the pronounced asymmetry of power between model deployers and model users. We therefore argue that model deployment should be subjected to regulatory scrutiny comparable to that applied to other regulated commercial practices.

To capture scenarios where an institution intentionally presents a misleading model to users, we introduce an extension of the PP framework, \textit{Decoupled} Performative Prediction (DPP), where the model influencing the data ($\theta_D$) is decoupled from the one being optimized ($\theta_M$). Therefore, we consider the decoupled performative risk.

\begin{equation}
    \mathcal{DPR}(\theta_D, \theta_M ) = \mathbb{E}_{z\sim \mathcal{D}(\theta_D)} [\ell(z;\theta_M)].
\end{equation}

Note that if $\theta_M = \theta_D = \theta$, then $\mathcal{DPR}(\theta,\theta) = \mathcal{PR}(\theta)$; i.e., users react to the same model that the institution deploys, and the setting reduces to the standard PP framework. Figure \ref{fig:DR-3d_plot} shows the optimization landscape that the $\mathcal{DPR(\theta_D, \theta_M )}$ spans for Example \ref{ex:pricing_luxury}.

This \textit{decoupled} scenario is arguably \textit{very} plausible, as there are clear incentives for an organization to pursue it. Optimizing over the expanded parameter space \((\theta_D, \theta_M)\) offers the potential for reduced risk, as we prove in Section~\ref{ssec:the_decoupled_optimum} and show experimentally in Section~\ref{sec:experiments}. Importantly, this behavior comes at little to no cost for the institution, as they control the model's training, deployment, and presentation, while users are compelled to respond to what they perceive as the true model.

\begin{figure}[t]
\begin{center}
\centerline{\includegraphics[width=0.6\columnwidth]{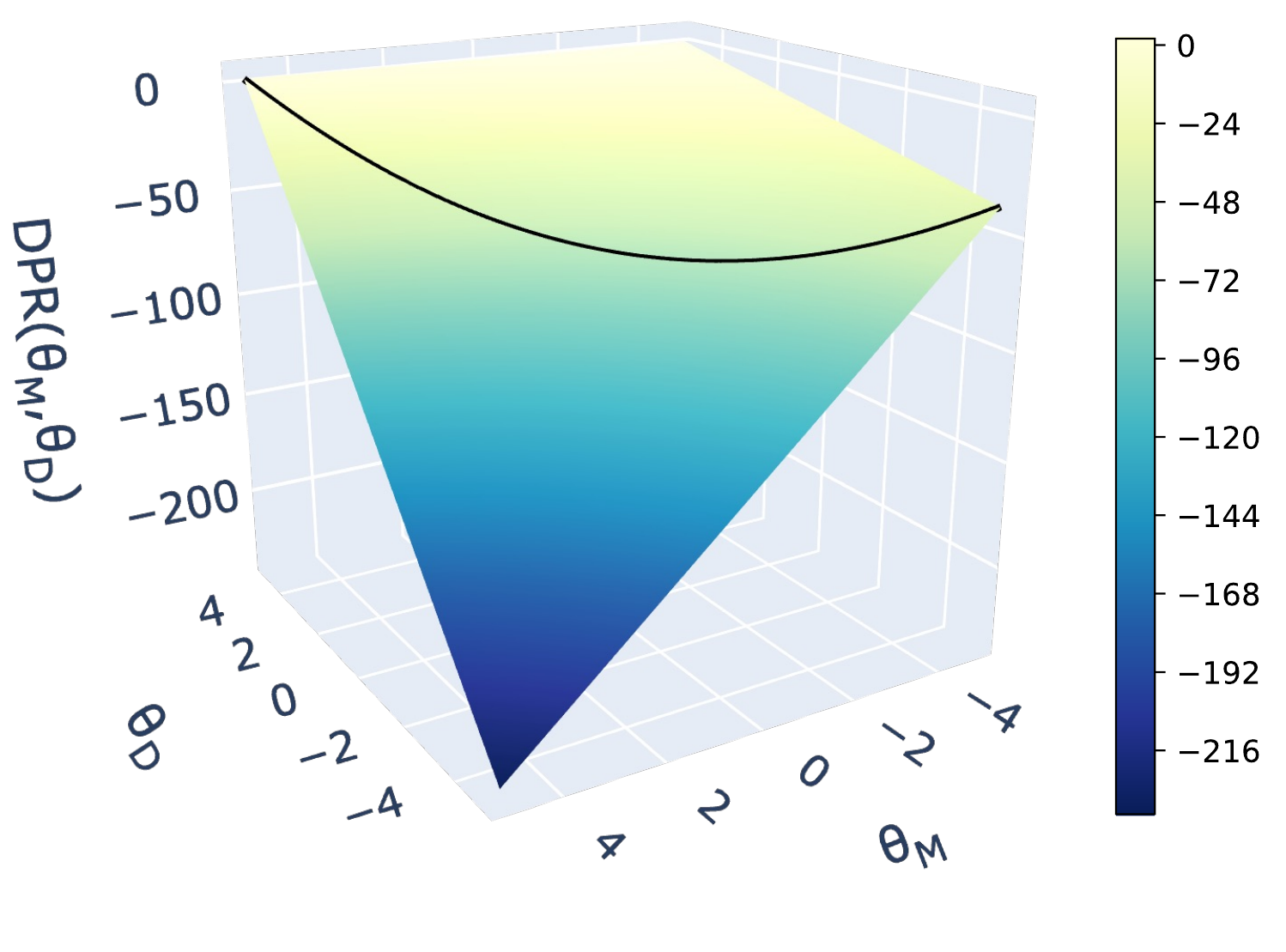}}
\caption{%
Visualization of the decoupled performative risk landscape of example \ref{ex:pricing_luxury} for $d=1$. This is the optimization space of DPP. The value of $\mathcal{DPR}(\theta_D,\theta_M)$ is represented through color intensity, with brighter colors corresponding to larger values and darker colors to smaller values. The black line represents the section of the surface where $\theta_D=\theta_M$, i.e., the standard performative risk.}%
\label{fig:DR-3d_plot}
\end{center}
\end{figure}

\begin{example}[Pricing with ads]
\label{ex:pricing_luxury}
Consider an airline that has one model for meta-search platforms to attract consumer attention, while the final price offered on their own booking platform is set by another one. We draw inspiration from the pricing example of \citet{izzo2021learn} to model this setting. In the resulting formulation, the model $\theta \in \mathbb{R}^d$ represents the vector of price deviations from a baseline price $\theta_0$ across $d$ products and the input $z \in \mathbb{R}^d$ corresponds to the demand for each product, e.g., the number of people that want to buy the product. The demand vector associated with the baseline price $\theta_0$ is denoted by $z_0$. A company wants to maximize their total revenue ($(\theta_0 + \theta)^Tz$), thus $\ell(z;\theta) = -(\theta_0 + \theta)^Tz$. The distribution map is defined by $\mathcal{D}_\mathrm{pr}(\theta)=\mathcal{N}(z_0-\varepsilon\theta; \Sigma)$ with $\varepsilon>0$, i.e. the demand linearly decreases as the price increases. Now, for the DPP scenario, we consider two price deviations for each product: $\theta_M$, which is used by the company to set the real prices and $\theta_D$, which is used for meta-search platforms, and therefore is the one that clients react to. In this example,
\begin{equation}
    \mathcal{DPR}(\theta_D,\theta_M) = \mathbb{E}_{z\sim\mathcal{N}(z_0 - \varepsilon\theta_D; \Sigma)}[-(\theta_0 + \theta_M)^Tz]. 
\end{equation}
\end{example}

\begin{example}[Money Lending Strategic Classification]
\label{ex:tricking-stragetic}
It extends the classic strategic classification setup \cite{peformativeprediction2020perdomo,HarMegPapWoo16}. Consider one bank that provides a website interface where users can interact with a credit model $\theta_D$, receiving immediate feedback after entering their original financial details $z_o\sim\mathcal{D}_o$. As the users could strategically game the system and change their input features with $x = x_o + \varepsilon\nabla_x f(x_0; \theta)$, the bank is unlikely to use this public model for real decisions. Instead, it is probable that they employ an internal model $\theta_M$, which is kept secret. In this case, 
\begin{equation}
    \mathcal{DPR}(\theta_D,\theta_M) = \mathbb{E}_{z_o\sim\mathcal{D}_o}[\ell(x_0 + \varepsilon\nabla_x f(x_0; \theta_D);\theta_M)].
\end{equation}

\end{example}

\subsection{Interest point in \textit{Decoupled} PP: \textit{decoupled} optimal point}
\label{ssec:the_decoupled_optimum} 

The minimum of the new optimization space has a risk that is potentially smaller than the performative risk achieved by the optimal point. Although \textit{decoupled stable} points exist in DPP, they are less compelling than the \textit{decoupled optimal} point, as there is one stable point for each possible output of the distribution map, i.e., for each $\theta_D$. For this reason, the focus of this section is only on the decoupled optimal point.
A detailed description of decoupled stable points is available in Appendix~\ref{appendix:decoupled-stable}.

%
\begin{definition}
The \textit{decoupled} optimum is defined as:
\begin{equation}
    (\theta_M^*, \theta_D^*) = \argmin_{(\theta_M, \theta_D)\in\Theta \times \Theta} \mathcal{DPR}(\theta_M, \theta_D)
\end{equation}
\end{definition}

\begin{example}
\label{ex:biased-coin-dpr}
We now consider the example \ref{ex:biased-coin} under DPP.
The institution's deployed decision model is $f_{\theta_M} = \frac{1}{2} + \theta_M x$, and the institution evaluates performance using the squared loss $\ell(x,y;\theta_M) = (y - f_{\theta_M}(x))^2$.
However, behavioral adaptation is governed by a potentially different disclosed model parameter $\theta_D$.
Accordingly, the underlying outcome is generated according to $Y|X\sim\text{Bernoulli}(\frac{1}{2}+\mu X + \varepsilon \theta_DX)$, so that the induced data distribution satisfies $(X,Y)\sim \mathcal{D}(\theta_D)$. 
Remember that in this decoupled setting, institutional decisions are made using $\theta_M$%
, while the data-generating process is shaped by $\theta_D$, explicitly separating the model that governs outcomes from the model that governs behavioral adaptation.
The decoupled optimal point is 

$$(\theta^*_D,\theta^*_M) = (1, \mu+\varepsilon),$$
which results in a risk of 
\begin{equation*}
    \mathcal{DPR}(\theta^*_D,\theta^*_M) = \frac{1}{4} - (\mu+\varepsilon)^2.
\end{equation*} 

The details of this derivation can be consulted in Appendix~\ref{appendix:details-example}. The remaining question is how this quantity compares to $\mathcal{PR}(\theta_{OP})$. We compute
\begin{align*}
    \mathcal{PR}(\theta_{OP}) - \mathcal{DPR}(\theta^*_D,\theta^*_M)
    = (\mu + \varepsilon)^2 - \frac{\mu^2}{1-2\varepsilon} 
    = \frac{\varepsilon\bigl(2\mu(1-\mu-2\varepsilon) + \varepsilon(1-2\varepsilon)\bigr)}{1-2\varepsilon}.
\end{align*}
Remember that under the standing assumptions, $\varepsilon>0$ and $\mu>0$. Moreover, since $\varepsilon < \frac{1}{2}-\mu$, it follows that $1-2\mu-2\varepsilon>0$. Therefore $1-2\mu-2\varepsilon>0$ and $1-2\varepsilon>0$. Consequently, the numerator and denominator of the above expression are non-negative, implying that, in this example,
\[
\mathcal{PR}(\theta_{OP}) \ge \mathcal{DPR}(\theta^*_D,\theta^*_M).
\]
\end{example}

However, this inequality is not specific to this example.
The decoupled performative formulation optimizes over a strictly larger space 
than the classical performative objective. As a consequence, the decoupled 
optimum can only improve (or match) the achieved risk. The following theorem quantifies the gap between the classical performative optimum and the decoupled optimum by providing lower and upper bounds.%

\begin{theorem}[Gap between performative and decoupled optima]
\label{th:gap-pr-dpr}
Assume that \(\ell(z;\theta)\) is (i) \(\mu\)-strongly convex in \(\theta\) for all \(z\), and (ii) \(L\)-Lipschitz in \(\theta\) for all \(z\).
Then, the decoupled optimum achieves a risk lower than (or at least equal to) the performative optimum, and the gap satisfies
\begin{equation}
    L \|\theta_D^* - \theta_M^*\|_{\ell_2}
    \;\ge\;
    \mathcal{PR}(\theta_{OP})
    -
    \mathcal{DPR}(\theta_D^*,\theta_M^*)
    \;\ge\;
    \frac{\mu}{2}
    \|\theta_{OP} - \theta_M^{\mathrm{BR}}(\theta_{OP})\|_{\ell_2}^2,
\end{equation}
where \(\theta_M^{\mathrm{BR}}(\theta_D)
\in
\arg\min_{\theta_M}
\mathcal{DPR}(\theta_D,\theta_M)\)
denotes the best-response model under the distribution induced by \(\theta_D\).
\end{theorem}

The proof can be found in Appendix~\ref{annex:proof-bounds}.

The lower bound depends on how far the performative optimum is from being a fixed point of the best-response operator $\|\theta_{OP} - \theta_M^{\mathrm{BR}}(\theta_{OP})\|^2_{\ell_2}$. If
\(
\theta_{OP} = \theta_M^{\mathrm{BR}}(\theta_{OP}),
\)
then the performative optimum is a fixed point of the best-response operator, i.e., it is also stable. In this case, the lower bound vanishes, and the theorem does not guarantee any strict improvement from decoupling. Nevertheless, optimality and stability do not generally coincide. Stable points need not be optimal, nor even close to optimal, unless additional structural assumptions are imposed \cite{peformativeprediction2020perdomo}. In practice, such assumptions are often not satisfied \cite{kabra2024limitations}.
Our theorem shows that whenever the performative optimum is not stable \(
\theta_{OP} \neq \theta_M^{\mathrm{BR}}(\theta_{OP})
\), decoupling achieves strictly lower risk \(
\mathcal{PR}(\theta_{OP}) > \mathcal{DPR}(\theta_D^*, \theta_M^*)
\).  

The upper bound is controlled by the disagreement
\(
\|\theta_D^* - \theta_M^*\|_{\ell_2}
\)
at the decoupled optimum. In particular, if
\(
\theta_D^* = \theta_M^*,
\)
then the upper bound is zero, and therefore the gap itself must be zero. This corresponds to the standard PP setting, where one assumes \(\theta_D = \theta_M\).

Taken together, the theorem shows that non-stability of the performative optimum is sufficient to guarantee a strict benefit from decoupling, while coincidence of the decoupled pair is sufficient to make the gap vanish. We believe that achieving a lower risk without any associated cost is a clear incentive for the institution to pursue the decoupled point.

\subsection{Convergence in \textit{Decoupled} PP}

\textit{Standard} PP has explored whether the performative optimum can be reached by discussing the convexity of $\mathcal{PR}(\theta)$ \cite{miller2021outside,cyffers2024optimal}. We do the same for the decoupled optimum and $\mathcal{DPR}(\theta_M,\theta_D)$. 

We base our analysis on the push-forward model $\varphi(\cdot;\theta)$ \cite{cyffers2024optimal}, as it is a very powerful mechanism for PP. It captures the effect of the shift on individual samples, which is the essence of performative prediction, e.g, when applying for a loan, what changes the original distribution are the actions performed by individuals. 

\begin{definition}[Push-forward model]
\label{def:push-forward}
    A distribution map, $\mathcal{D}(\cdot)$, has a push-forward model, $\varphi(\cdot;\theta):\mathbb{R}\rightarrow\mathbb{R}$ if:
    \begin{equation}
        z \sim \mathcal{D}(\theta)\quad\iff\quad z = \varphi(z_o;\theta), \quad z_0 \sim \mathcal{D}_0 
    \tag{A1}
    \label{eq:push-forward}
    \end{equation}
    where $\mathcal{D}_0$ is a base distribution.
\end{definition}
Both location-scale families and distribution maps arising from strategic classification can be naturally represented with the push-forward model.
Additionally, many distribution maps with a functional form can be reparameterized such that they have a base distribution that is independent of $\theta$ and a transformation function that depends on $\theta$.
This is, for example, the case for any functional form that is defined in terms of Gaussian distributions and uniform distributions.

Moreover, we assume that this push-forward model is an affine function with respect to $\theta_D$.

\begin{definition}[Affinity]
    A function $f:\mathbb{R}^n\to\mathbb{R}^m$ is \emph{affine} if
    \begin{equation}
        f(\lambda_1x_1 + \ldots + \lambda_nx_n) = \lambda_1 f(x_1) + \ldots + \lambda_n f(x_n)
    \tag{A2}
    \label{eq:affine-function}
    \end{equation}
    where $\lambda_1 + \ldots + \lambda_n=1$.
\end{definition}

Location-scale families provide a familiar example of affine push-forward maps, but they are not the only possibility. 
More generally, one can consider distribution maps in which a base distribution (e.g., Gaussian or uniform) is transformed by an affine function of the parameter $\theta$.
In contrast, strategic classification models yield an affine push-forward map only under linear decision rules.
When the classifier is nonlinear, the model parameters typically enter users' best-response transformations in a nonlinear way, which in turn induces a distribution map that is no longer affine in $\theta$.

One standard assumption for the field is convexity of the loss function $\ell(z, \theta)$ over $\theta$ \cite{peformativeprediction2020perdomo} and over $z$ \cite{miller2021outside}. In our case, we consider joint convexity.

\begin{definition}[Jointly convex]
    A function $f: \mathbb{R}^d\times\mathbb{R}^d\to\mathbb{R}$ is \emph{jointly convex} if
    \begin{equation}
        f(\lambda x + (1-\lambda)x';\lambda y + (1-\lambda)y') \leq \lambda f(x;y) + (1-\lambda) f(x';y')~.
    \tag{A3}
    \label{eq:joint-convexity}
    \end{equation}
\end{definition}

\begin{theorem}
\label{th:convexity-dpr}
If $\ell(z;\theta_M)$ is jointly convex (\ref{eq:joint-convexity}) and $\mathcal{D}(\theta_D)$ follows the push-forward model (\ref{eq:push-forward}) where $\varphi(z_o;\theta_D)$ is affine in $\theta_D$ (\ref{eq:affine-function}), then $\mathcal{DPR}(\theta_M,\theta_D)$ is jointly convex.
\end{theorem}

Note that joint convexity implies that the decoupled risk has a minimum that is reachable by gradient descent.

\begin{corollary}
    If $\mathcal{DPR}(\theta_M,\theta_D)$ is jointly convex, then all local minima are global minima and gradient descent converges to the global minimum.
\end{corollary}

Therefore, if we apply GD to the decoupled risk, we converge to the decoupled optimum if the assumptions are met. We call this algorithm Decoupled Performative Gradient Descent (DPerfGD), and it is shown in Algorithm \ref{alg:DPerfGD}.  

\begin{algorithm}[t]
\caption{Decoupled Performative Gradient Descent (DPerfGD)}\label{alg:DPerfGD}
\begin{algorithmic}
\Input $\theta_0$ (initial model parameters)
\State $t=0$ 
\State $\theta^{(0)} = \theta_0$ 
\While{\text{not converged}} 
\State $\text{Draw}\ z_i^{(t)} \sim \mathcal{D}_0, i = 1, \dots, n$ 
\State Compute $\nabla_{\theta_M} \, \mathcal{DPR}$ according to Eq.~\eqref{eq:model-derivative}
\State Compute $\nabla_{\theta_D} \, \mathcal{DPR}$ according to Eq.~\eqref{eq:data-derivative}
\State $(\theta^{(t+1)}_D, \theta^{(t+1)}_M) \gets (\theta^{(t+1)}_D, \theta^{(t+1)}_M) - \eta (\nabla_{\theta_D} \mathcal{DPR}, \nabla_{\theta_D} \mathcal{DPR})$
\State $t \gets t +1$
\EndWhile
\end{algorithmic}
\end{algorithm}

\newpage

To apply gradient descent, one needs to compute the corresponding gradients. To calculate the gradient with respect to $\theta_D$, one can use the reparametrization trick assuming that the distribution map can be written using a push-forward model\footnote{We can also use the REINFORCE trick if the probability density function of the image of the distribution map is known, but we consider this more unusual in PP.}:
\begin{align}
    \nabla_{\theta_M}\,\mathcal{DPR}(\theta_M, \theta_D) &= \mathbb{E}_{z \sim \mathcal{D}(\theta_D)} \left[\frac{\partial}{\partial \theta_M}  \ell(z; \theta_M) \right], 
    \label{eq:model-derivative}\\
    \nabla_{\theta_D}\,\mathcal{DPR}(\theta_M, \theta_D) &= \nabla_{\theta_D}\, \mathbb{E}_{z_0 \sim \mathcal{D}_0} \left[ \ell(\varphi(z_o;\theta_D); \theta_M) \right] \nonumber \\  
    &= \mathbb{E}_{z_0 \sim \mathcal{D}_0} \left[ \left.\frac{\partial \ell(z; \theta_M)}{\partial z}\right|_{z=\varphi(z_o;\theta_D)} \cdot\frac{\partial \varphi(z_o;\theta_D)}{\partial \theta_D} \right]. 
    \label{eq:data-derivative}
\end{align}

The implementation of this algorithm can be found in the performativeGYM library \cite{bautiste2026performativepredictionmade}.

\section{Deception cost: A quantitative measure of the ease with which institutions can sustain adversarial deployment
practices}
\label{sec:deceptive-cost}

So far, we have shown how an institution that controls model deployment can engage in deceptive behavior for its own benefit at essentially no cost.
However, in reality, institutions cannot deceive users without limit.
If the manipulation is too blatant, users will catch on.
Thus, there is an implicit cost (e.g., the risk of losing users) for the institution.
We call this the \emph{deception cost} and model it as being proportional to $\|\theta_D - \theta_M\|_{\ell_2}$, i.e., proportional to how far apart the true model and the user-facing model are.

As discussed earlier, asymmetries of power also arise in other domains, such as false advertising and consumer protection. In these contexts, institutions often justify deceptive practices by appealing to a formal notion of user choice, arguing that consumers retain the freedom to opt out and refrain from using the product.
This argument is frequently invoked to resist stronger regulatory intervention. It relies on the assumption that markets discipline deceptive behavior through user exit. However, this assumption often fails in practice, which is precisely why many sectors are subject to strong regulation aimed at protecting users. In many cases, users face substantial switching costs, limited competition, informational asymmetries, or institutional dependencies that make exit infeasible or prohibitively costly. Similar dynamics arise in learning settings: individuals subject to algorithmic decision-making in credit approval, hiring, insurance pricing, content moderation, or access to public services cannot meaningfully opt out without incurring significant economic or social harm. In such contexts, the formal availability of choice does not translate into effective consent or meaningful market discipline.

In our setting, this institutional reasoning manifests through the assumption that users can simply choose not to rely on a disclosed model. Institutions may therefore claim that publishing a model \(\theta_D\) that appears aligned with the deployed model \(\theta_M\) is sufficient. We argue that such behavior constitutes a form of regulatory arbitrage: compliance is signaled through disclosure, while incentives are optimized through deployment.

Throughout this section, we formalize this behavior by considering institutions that internalize the risk of detection, reputational harm, or regulatory scrutiny associated with deceptive deployment practices, and therefore optimize under a deception cost that penalizes misalignment between disclosure and deployment. We show that such institutional incentives alone are insufficient to adequately protect users.
%
%
Let us consider a setting where the institution seeks to keep the disclosed model $\theta_D$ close to the true model $\theta_M$, either due to potential consequences of deception or to external regulatory constraints. Then, the institution would have to solve the following constrained optimization problem:
\begin{align}
(\theta_D^\star, \theta_M^\star)
\in
\argmin_{(\theta_D, \theta_M)\in\Theta\times\Theta}
\;&
\mathcal{DPR}(\theta_D,\theta_M)
\nonumber\\
\text{s.t.}\qquad
&
\|\theta_D-\theta_M\|_{\ell_2}^2
\le c,
\quad
c \ge 0.
\label{eq:optimization-problem-deception}
\end{align}
This constrained optimization problem can be transformed into an unconstrained problem 
by introducing a regularization term $\lambda$ as follows: 
%
\begin{equation}
\label{eq:optimize-soft-constraint}
    \argmin_{(\theta_D, \theta_M) \in \Theta \times \Theta} \mathcal{DPR}(\theta_D, \theta_M) + \lambda\|\theta_D-\theta_M\|^2_{\ell_2},  \quad\text{with } \lambda >0.
\end{equation}
The regularization term $\lambda$ may, for example, reflect the magnitude of the institution’s anticipated risk of user abandonment due to deceptive behavior or the constraint imposed by an external regulator.

Regardless of whether we use a hard constraint or a regularized term, the mechanism of the deception cost biases the optimization towards exploring regions that are close to the ``diagonal'' of the search space where $\theta_M=\theta_D$.
In the extreme case, where $c=0$ or $\lambda\to\infty$, this collapses the setting back into the ordinary performative prediction setup.
For the non-extreme case with finite $c$ and $\lambda$, the effect is a restriction on the search space.
From Theorem~\ref{th:gap-pr-dpr}, we have the upper bound on the risk gap between the performative optimum and the decoupled optimum $    L \|\theta_D^* - \theta_M^*\|_{\ell_2}
    \;\ge\;
    \mathcal{PR}(\theta_{OP})
    -
    \mathcal{DPR}(\theta_D^*,\theta_M^*)~.
$
Thus, a hard constraint such as $\|\theta_D^*-\theta_M^*\|_{\ell_2} \le \sqrt{c}$ from equation \eqref{eq:optimization-problem-deception}, directly affects the achievable gap.

Figure~\ref{fig:algorithm-trajectory} illustrates the optimization landscape and the trajectory of DPerfGD when the deception cost is incorporated as a regularization term in Example~\ref{ex:pricing_luxury}.

\begin{figure}[t]
\begin{center}
\centerline{\includegraphics[width=0.8\paperwidth]{figures/algorithm_trajectory/viz_pricing_v2.pdf}}
\caption{%
Visualization of the trajectory of DPerfGD in the decoupled risk landscape for the Example \ref{ex:pricing_luxury}. \textbf{A} without taking into account any deception cost, i.e., top view of 3D surface in Figure~\ref{fig:DR-3d_plot}; \textbf{B}, \textbf{C}, \textbf{D}, \textbf{E} with the deception cost as a regularizer (soft constraint, eq. \ref{eq:optimize-soft-constraint}) with strength $\lambda={2,4,8,128}$ respectively. The black line represents the section of the surface where $\theta_D=\theta_M$, i.e., the optimization space of \textit{standard} PP. DPerfGD converges to the decoupled optimum. Increasing the regularization strength modifies the loss landscape and pulls the decoupled optimum towards the performative risk. Nevertheless, at some point, the regularization strength makes optimization unstable, so for \textbf{D} and \textbf{E}, we have reduced the learning rate.}%
\label{fig:algorithm-trajectory}
\end{center}
\end{figure}

\subsection{Perceived deception cost}
Defining the deception cost in the parameter space, $\Theta$, is mathematically convenient, but it likely does not capture the experience of the users.
It is hard for users to compute distances on $\theta$.
Instead, they will notice differences in model output over the deployment data distribution.
To capture this, we introduce an alternative notion of \emph{perceived} deception cost, measured by $\mathbb{E}_{z\sim D(\theta_D)}[|f_{\theta_D}(z)-f_{\theta_M}(z)|^p]$, where $f_\theta(x)$ is the output of model $f$ with parameters $\theta$ to input $x$ and $p\in\{1, 2\}$.
We can again formulate the optimization problem with a hard constraint or with a regularization term:

\begin{align}
\label{eq:optimization-problem-perceived-deception}
(\theta_D^\star, \theta_M^\star)
\in
\argmin_{(\theta_D, \theta_M)\in\Theta\times\Theta}
\;&
\mathcal{DPR}(\theta_D,\theta_M)
\nonumber\\
\text{s.t.}\qquad
&
\mathbb{E}_{z\sim D(\theta_D)}
\!\left[
|f_{\theta_D}(z)-f_{\theta_M}(z)|^p
\right]
\le c',
\quad
c'\ge 0.
\end{align}%
or
\begin{equation}
    \argmin_{(\theta_D, \theta_M) \in \Theta \times \Theta} \mathcal{DPR}(\theta_D, \theta_M) + \lambda^{\prime}\mathbb{E}_{z\sim D(\theta_D)}[|f_{\theta_D}(z)-f_{\theta_M}(z)|^p],\quad \text{with }\lambda'>0.
\end{equation}
As with the deception cost over model parameters, the effect here is likewise to restrict the optimization space.

The formulation with the regularization term is easier to optimize, but the formulation with the hard constraint allows us to give a new bound on the gap between the performative and decoupled optima (cf.\ Theorem~\ref{th:gap-pr-dpr}).
For this bound, we assume a supervised learning setup where the loss function is defined on the labels $y$ and the model outputs $\hat y$: $\mathcal{L}(\hat y, y)$. We furthermore assume that $\mathcal{L}(\hat y, y)$ is Lipschitz in $\hat y$, which is, e.g., the case for the cross-entropy loss with logits (see Appendix~\ref{appendix:ce-lipschitz}) and the L1 loss.

\begin{theorem}[Gap between performative and decoupled optima with deception cost]
\label{th:gap-pr-dpr-dc}
Assume that \(\mathcal{L}(\hat y, y)\) is \(L\)-Lipschitz in \(\hat y\) for all \(y\).
Then, under the setting of perceived deception cost with $\mathbb{E}_{(x,y)\sim D(\theta_D)}[|f_{\theta_D}(x)-f_{\theta_M}(x)|] \le c'$, the risk gap between the performative optimum and the decoupled optimum is bounded by
\begin{equation}
    c'L\;\ge\;\mathcal{PR}(\theta_{OP})
    -
    \mathcal{DPR}(\theta_D^*,\theta_M^*).
\end{equation}
\end{theorem}
The threshold $c'$, representing the maximum allowed \emph{perceived} deception cost, therefore directly determines how much utility (i.e., risk reduction) a deployer can gain through user deception. The proof can be found in Appendix~\ref{appendix:deception-cost-risk}.




\begin{figure}[t]
\begin{center}
\centerline{\includegraphics[width=\columnwidth]{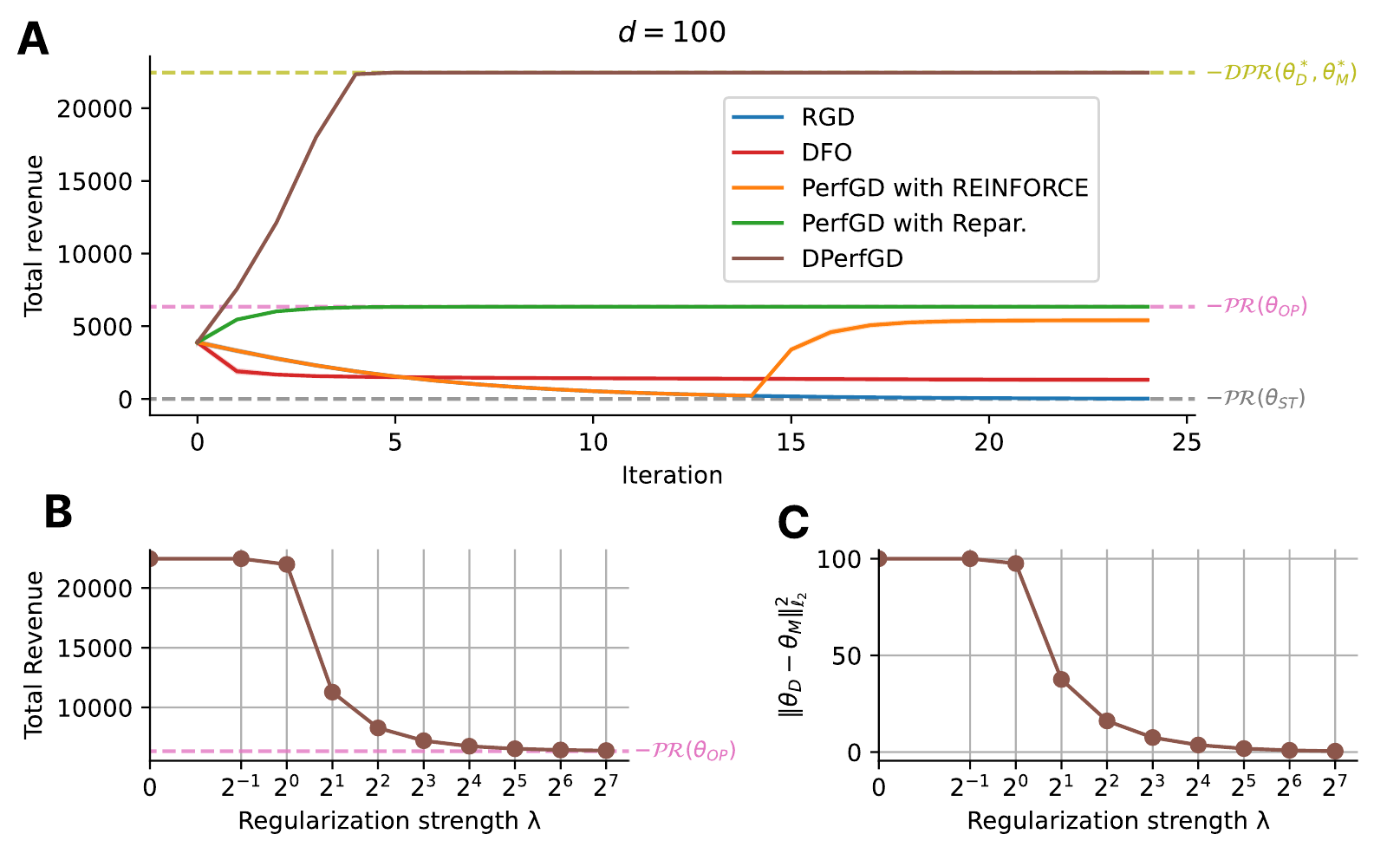}}
\caption{%
\textbf{A} Total revenue (higher is better) for $100$ products, i.e., $d=100$, on the \textbf{pricing} dataset. Optimizing for the \textit{decoupled} optimum retrieves more revenue than reaching the standard optimal point of PP. DPerfGD converges to the decoupled optimum. The rest of the algorithms converge to their corresponding solutions, as shown by \citet{izzo2021learn}. \textbf{B} Total revenue (value at the last iteration) of DPerfGD for several regularization strengths. \textbf{C} shows how the deception cost decreases increasing $\lambda$. Although this figure shows results for $d=100$, note that we chose the same regularization strengths as in Figure \ref{fig:algorithm-trajectory} (where $d=1$). Observe that \textbf{B} can be seen as a summary of \textbf{A} if we include in \textbf{A} the experiments for several regularization strengths. 
}
\label{fig:pricing-exp}
\end{center}
\end{figure}
\section{Experiments}
\label{sec:experiments}

In this section, we empirically evaluate settings in which an institution strategically deceives users by decoupling the model that governs its decisions from the model that shapes user behavior.\footnote{The code for the experiments is part of the performativeGYM library \cite{bautiste2026performativepredictionmade} and can be found here \url{https://github.com/wearepal/performativeGYM/}} We compare outcomes under such behavior, modeled through the DPP framework, with those obtained when the institution does not engage in deception, corresponding to standard PP. We report both institutional performance (the objective optimized by the institution) and the deception cost, which serves as a proxy for user-facing harm, across the different optimization approaches discussed, including \textsc{DPerfGD} and constrained optimization over the deception cost.

Figure~\ref{fig:pricing-exp} presents results on the synthetic pricing dataset introduced in Example~\ref{ex:pricing_luxury}, with $d=100$ products. We evaluate total revenue, measured as $-\mathcal{PR}(\theta)$ under standard PP and as $-\mathcal{DPR}(\theta_D,\theta_M)$ under DPP. This setting admits closed-form solutions for the stable point, the performative optimum, and the decoupled performative optimum, allowing us to directly compare algorithmic behavior against known ground-truth solutions and to isolate the effects of decoupling and deception. Consult Appendix~\ref{appendix:pricing-details} for the derivation.

We consider the following baselines. Repeated Gradient Descent (RGD) \cite{peformativeprediction2020perdomo}, which converges to a stable point. Two variants of Performative Gradient Descent (PerfGD) are also included: PerfGD with REINFORCE \cite{izzo2021learn}, which requires a retraining warm-up phase (hence its initial behavior closely resembles RGD), and PerfGD with the reparameterization trick \cite{cyffers2024optimal}, referred to as PerfGD (reparam.). Both variants converge to the performative optimum. Finally, we include a Derivative-Free Optimization (DFO) method \cite{flaxman2005online}. Although not specifically designed for performative prediction, it serves as a simple black-box baseline.

\begin{figure}[t]
\begin{center}
\centerline{\includegraphics[width=\columnwidth]{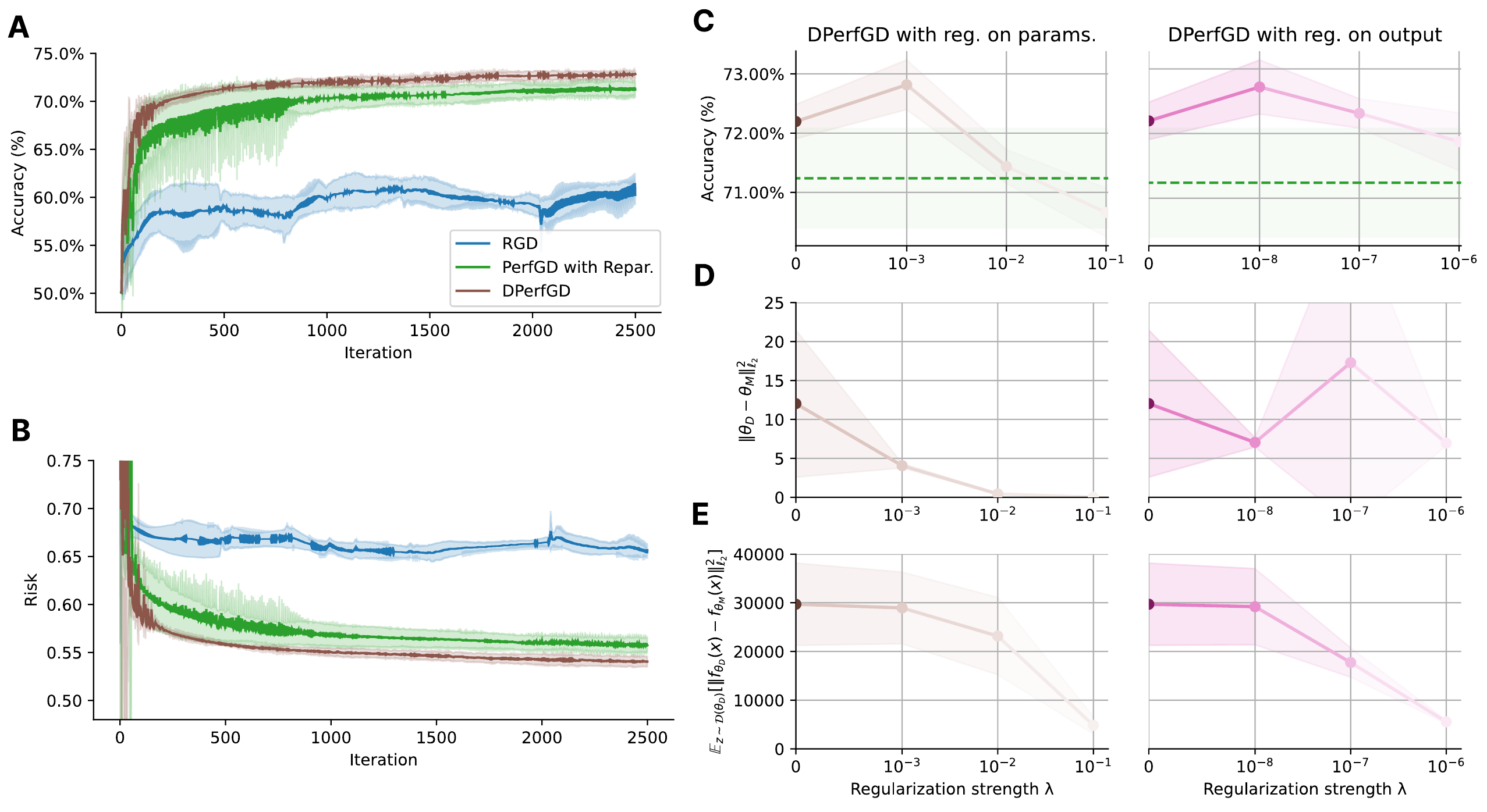}}%
\caption{%
\textbf{A} shows the accuracy on the \textit{Give Me Some Credit} dataset \cite{GiveMeSomeCredit} (money lending) with a strategic classification distribution map with a 2-layer NN (100 hidden neurons) for multiple algorithms.
\textbf{B} shows the loss and \textbf{C}, \textbf{D}, \textbf{E} show the tradeoff between regularization strength and multiple metrics. The left column shows the metrics when adding a $\|\theta_D-\theta_M\|^2_{\ell_2}$ penalty (reg. on params.) and the right column when adding a $\|f_{\theta_D}(x) - f_{\theta_M}(x)\|^2_{\ell_2}$ penalty (reg. on output). In this case, we opted to leave out the big $\lambda$s that lead to unstable training to facilitate the interpretation of the plots. 
PerfGD outperforms RGD, as PerfGD is able to reach the \textit{standard} performative optimum, whereas RGD can only find the stable point. Nevertheless, as expected, the accuracy is better when optimizing in the DPP framework. 
}
\label{fig:credit-exp}
\end{center}
\end{figure}

Figure~\ref{fig:credit-exp} reports results on the real-world \textit{Give Me Some Credit} dataset with strategic classification, as introduced in Example~\ref{ex:tricking-stragetic}. This tabular dataset contains financial features of loan applicants ($x \in \mathbb{R}^{11}$) and a binary outcome indicating whether a loan was repaid ($y=0$) or defaulted ($y=1$). While this dataset, equipped with a distribution map based on strategic classification \cite{HarMegPapWoo16}, has become a standard benchmark in performative prediction, we extend prior work by considering a two-layer multilayer perceptron as a model instead of just logistic regression. We use the binary cross-entropy loss to train the two-layer neural network.

Unlike the previous synthetic experiment, this dataset does not admit closed-form expressions for the stable or optimal performative points. Therefore, we report accuracy as the primary performance metric. Additionally, we do not apply the REINFORCE-based variant of performative gradient descent in this setting, as it requires explicit knowledge of the distribution map governing user responses, which is unavailable in real-world data.

Further experimental details for both setups are provided in Appendix~\ref{appendix:experimental-details}.

Figure \ref{fig:pricing-exp}\textbf{A} and Figures \ref{fig:credit-exp}\textbf{A}, \ref{fig:credit-exp}\textbf{B} show that DPerfGD consistently outperforms PerfGD. This behavior is expected in light of Theorem \ref{th:gap-pr-dpr}, which establishes that the risk of the decoupled optimum is always lower than or equal to the risk of the standard performative optimum, i.e.,
$\mathcal{DPR}(\theta_D^*,\theta_M^*) \le \mathcal{PR}(\theta_{OP})$. These results confirm that an institution controlling the model can achieve a strictly better outcome by adopting a deceptive strategy.

Figures \ref{fig:pricing-exp}\textbf{B} and \ref{fig:credit-exp}\textbf{C} illustrate that introducing a penalty term reduces institutional benefits when the strength is sufficiently large. In the latter, we can also see how it performs below the performative risk of the standard optimum when the regularization strength is sufficiently large. Note that this statement refers to the accuracy metric, which is a usage metric and does not necessarily vary monotonically with the loss. Therefore, institutions have an incentive to select small values of $\lambda$ to avoid such performance degradation. For values of $\lambda$ that still yield an acceptable performance, the associated deception cost remains very large (see Figure \ref{fig:pricing-exp}\textbf{C} and \ref{fig:credit-exp}\textbf{D},\textbf{E}). Consequently, institutions have little incentive to adopt large regularization strengths, leaving users largely unprotected. Note that the users do not know the \textit{true} deception cost as they do not have access to $\theta_M$. These observations highlight the need for additional explicit regulatory constraints to protect users from strategically deceptive behavior from the institutions that control model deployment.

\section{Conclusion}

In this paper, we introduce Decoupled Performative Prediction (DPP), a framework that relaxes the standard assumption that the model deployed by an institution is identical to the model to which the environment responds. This relaxation enables the study of realistic and practically relevant deployment settings, but it also exposes a critical vulnerability:
institutions can strategically shape behavior through misaligned or deceptive disclosures, with potentially severe consequences for affected individuals.
By explicitly modeling scenarios in which users lack transparency into decision-making systems, we formalize how informational asymmetries can generate systematic risks, misaligned incentives, and harmful outcomes.

To quantify these dynamics, we introduce the notion of deception cost, which 
characterizes the ease with which institutions can sustain adversarial disclosure strategies and the resulting burden imposed on users.
%
Our analysis complements the performative power framework of \citet{performativepower2022hard} by identifying a highly plausible and 
previously unmodeled mechanism through which institutions can exploit performative feedback, that is not merely to cope with distributional shifts, but to benefit from them.
This challenges the common assumption that performative effects are necessarily detrimental to institutional objectives.

More broadly, our results elevate disclosure from a matter of transparency to a foundational technical design decision that governs learning dynamics, risk, and welfare in performative systems.
These findings underscore an urgent need for regulatory oversight, auditing standards, and accountability mechanisms that treat disclosure as a central object of governance rather than a peripheral compliance requirement.
%

\paragraph{Limitations and future work.}
As in prior work on PP, our framework relies on specifying a distribution map that models how the environment responds to model deployment. In practice, these feedback mechanisms may involve complex behavioral or strategic dynamics and can be difficult to characterize precisely. Accordingly, our results should be interpreted as providing a principled framework for reasoning about the risks that may arise from interactions between users and the model. Future work could deepen our analysis of settings in which disclosed and deployed models differ due to strategic behavior, particularly by incorporating richer forms of institutional adaptation, heterogeneous user responses, and dynamic regulatory environments.

In addition, several practically relevant use cases of decoupled deployment remain for future work. For example, users may query the system repeatedly and may learn about its behavior over time.\footnote{\citet{ghalme2021strategic} study this setting in the context of strategic classification; extending such analyses to PP remains an interesting direction for future work.} Another interesting setting can be when the relationship between the disclosed and real model is constrained, extending our notion of deception cost to richer forms of limited or regulated disclosure. Finally, it would be interesting to analyze \emph{partial disclosure} regimes in which only a subset of model parameters or features is truthfully revealed. These extensions may help better understand how informational constraints shape performative dynamics in real-world systems.

Overall, we hope this work motivates future research on technical and legal safeguards that enable individuals, communities, and regulators to detect, contest, and protect against deceptive deployment practices.

\section*{Generative AI Usage Statement}

We have used ChatGPT 5.2 to support the styling (fluency and grammar) of the text during the preparation of the manuscript. 

\section*{Acknowledgements}

We thank the reviewers for their thoughtful and constructive feedback. This project has been directly supported by the Basque Government (BERC 2026–2029 and IT2109-26); by the Spanish Ministry of Science and Innovation \\(MICIN/AEI/10.13039/501100011033) through project PID2022-137442NB-I00 and the BCAM Severo Ochoa accreditation (PRE2022-102655 and CEX2021-001142-S); and by the European Union’s Horizon Europe Programme (Grant Agreement No. 101120763 – TANGO and No. 101123000 – Act.AI). The views and opinions expressed are those of the author(s) only and do not necessarily reflect those of the granting authorities. Neither the European Union nor the granting authorities can be held responsible for them.

\newpage
\bibliographystyle{unsrtnat}
\bibliography{sample-base}


\clearpage
\appendix

\newpage

\section{From traditional Machine Learning to extended Performative Prediction}%
\label{appendix:illustration}
To illustrate the differences between classical ML, PP and DPP, we include Fig.~\ref{fig:teaser}.

\begin{figure}[t]
\begin{center}
\centerline{\includegraphics[width=\columnwidth]{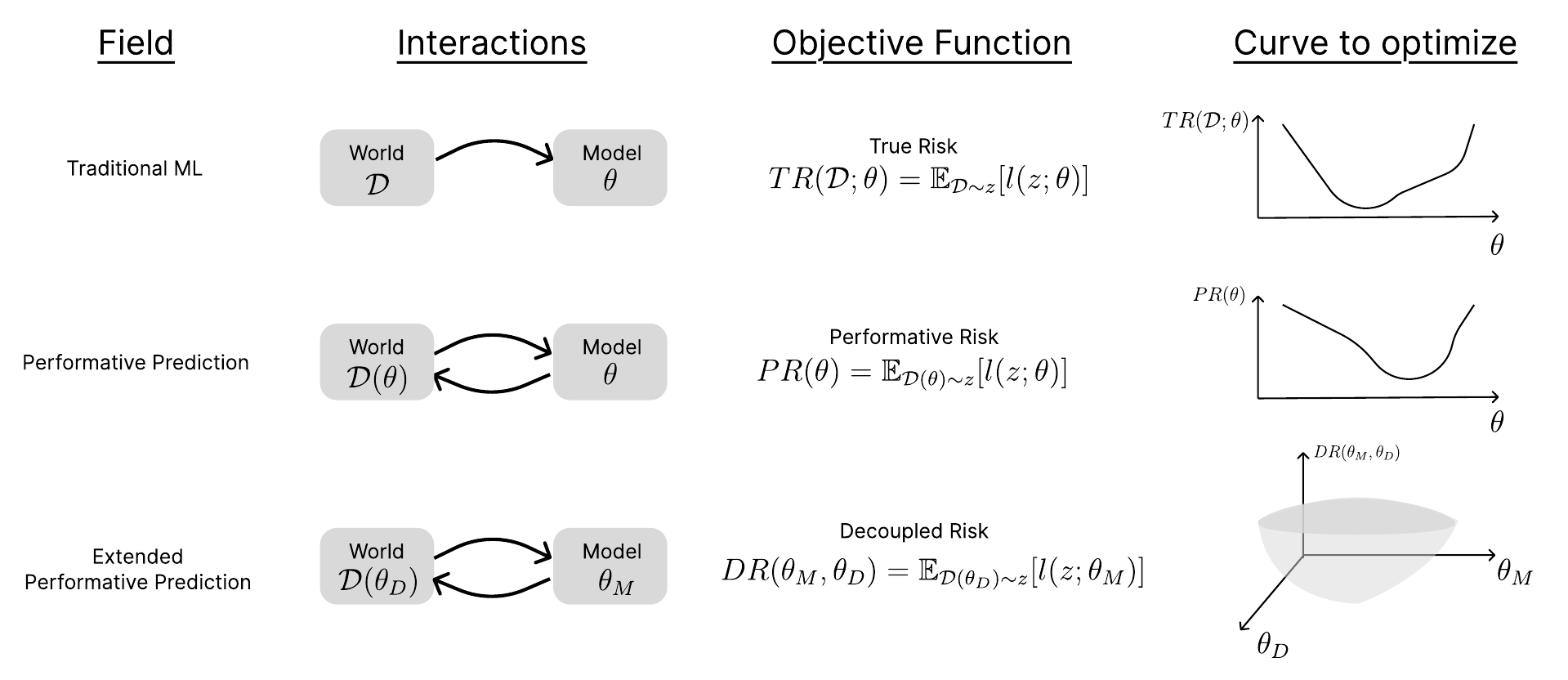}}
\caption{PP's formulation extends the classical risk by incorporating the dependency of the data distribution on the model parameters. Extended PP decouples the parameters. Note that \textit{standard} PP is \textit{extended} PP for $\theta_M = \theta_D = \theta$}
\label{fig:teaser}
\end{center}\end{figure}

\section{Details of the PerfGD variants}%
\label{appendix:perfgd-variants}

\citet{izzo2021learn} proposed PerfGD using the REINFORCE method to calculate the performative gradient. Nevertheless, \citet{cyffers2024optimal} realised that it can be calculated with the reparametrization method. This section has the mathematical details of both methods.

\subsection{Performative Gradient Descent (PerfGD) with REINFORCE}\label{appendix:perfgd-reinforce}
Uses the fact that the gradient of the likelihood of a random variable is the same likelihood times the gradient of the log likelihood $\nabla_{\theta} p_{\theta}(z)= p_{\theta}(z)\nabla_{\theta}\log p_{\theta}(z)$. 
\begin{align}
\label{eq:reinforce}
 \nabla_\theta PR(\theta) &= \nabla_\theta \int \ell(z;\theta) p_{\mathcal{D}(\theta)}(z) dz \nonumber \\
&= \int \frac{\partial \ell(z;\theta)}{\partial \theta}   p_{\mathcal{D}(\theta)}(z)dz + \int  \ell(z;\theta)  \frac{\partial p_{\mathcal{D}(\theta)}(z)}{\partial \theta}dz\nonumber\\
& = \int \frac{\partial \ell(z;\theta)}{\partial \theta}   p_{\mathcal{D}(\theta)}(z)dz + \int  \ell(z;\theta)   \frac{\partial \log p_{\mathcal{D}(\theta)}(z)}{\partial \theta} p_{\mathcal{D}(\theta)}(z)dz \nonumber \\
&=\mathbb{E}_{z \sim \mathcal{D}(\theta)} \left[\frac{\partial \ell(z; \theta)}{\partial \theta} \right] + \mathbb{E}_{z \sim \mathcal{D}(\theta)} \left[ \ell(z; \theta)\frac{\partial \log p_{\mathcal{D}(\theta)}(z)}{\partial \theta} \right] \nonumber\\
&=\mathbb{E}_{z \sim \mathcal{D}(\theta)} \left[\frac{\partial \ell(z; \theta)}{\partial \theta} + \ell(z; \theta)\frac{\partial \log p_{\mathcal{D}(\theta)}(z)}{\partial \theta} \right].
\end{align}

\subsection{PerfGD with reparametrization}\label{appendix:perfgd-rp}
Uses a deterministic function that is dependent in a base distribution and encodes the transformation caused by the parameter. Therefore, the expectation depends on the base-distribution only. In the case of PP, this is achievable by defining a base distribution \(\mathcal{D}(\theta)\) that captures the samples before performativity and a push-forward model that defines the transformation of each sample due to performativity $z = \varphi(z_o, \theta)$. We can then use the multivariate chain rule to calculate the performative gradient:

\begin{align}
\label{eq:reparam-trick}
    \nabla_\theta PR(\theta) &= \nabla_\theta \mathbb{E}_{z \sim \mathcal{D}(\theta)} \big[ \ell(z; \theta)\big] \nonumber \\
    & = \nabla_\theta \mathbb{E}_{z_o \sim \mathcal{D}_o} \big[ \ell(\varphi(z_o, \theta); \theta)\big] \nonumber \\
    &=  \mathbb{E}_{z_o \sim \mathcal{D}_o} \Big[ \nabla_\theta \ell(\varphi(z_o, \theta); \theta)\Big] \nonumber \\
    & = \mathbb{E}_{z_0 \sim \mathcal{D}_0} \left[ \left.\frac{\partial \ell(z; \theta)}{\partial \theta}\right|_{z=\varphi(z_0; \theta)}  + \left.\frac{\partial \ell(z; \theta)}{\partial z}\right|_{z=\varphi(z_0; \theta)} \frac{\partial \varphi(z_0;\theta)}{\partial \theta} \right].
\end{align}

\section{Theorems' proofs}
In this appendix, we simplify notation by writing $\|\cdot\|$ for the $\ell_2$ norm.

\subsection{Bounds proof}
\label{annex:proof-bounds}

Let \(\theta_M^{\mathrm{BR}}(\theta_D)\) denote the best-response (BR) model that is deployed when the population reacts to \(\theta_D\), i.e.,
\begin{equation}
    \theta_M^{\mathrm{BR}}(\theta_D)
    \in
    \argmin_{\theta_M} \mathcal{DPR}(\theta_D,\theta_M)
    =
    \argmin_{\theta_M}
    \mathbb{E}_{z\sim\mathcal{D}(\theta_D)}[\ell(z;\theta_M)].
\end{equation}

\begin{theorem}[Lower bound]
Assume that \(\ell(z;\theta)\) is \(\mu\)-strongly convex in \(\theta\) for all \(z\). Then
\begin{equation}
    \min_{\theta} \mathcal{PR}(\theta)
    -
    \min_{\theta_D,\theta_M} \mathcal{DPR}(\theta_D,\theta_M)
    \ge
    \frac{\mu}{2}
    \|\theta_{OP} - \theta_M^{\mathrm{BR}}(\theta_{OP})\|^2,
\end{equation}
where \(\theta_{OP} \in \argmin_\theta \mathcal{PR}(\theta)\).
\end{theorem}

\begin{proof}

Since \(\ell(z;\theta)\) is \(\mu\)-strongly convex in \(\theta\) for all \(z\), its expectation taken over $z\sim\mathcal{D}(\theta_D)$
\[
 \mathcal{DPR}(\theta_D,\theta_M) = \mathbb{E}_{z \sim \mathcal{D}(\theta_{D})}[\ell(z;\theta_M)]
\]
is also \(\mu\)-strongly convex in \(\theta_M\).

Since $\theta^{BR}_M$ is a minimizer of $\mathcal{DPR}$ with respect of $\theta_M$, strong convexity gives the inequality.

\[
\mathcal{PR}(\theta_{OP}) = \mathcal{DPR}(\theta_{OP},\theta_{OP})
\ge
\mathcal{DPR}(\theta_{OP},\theta_M^{\mathrm{BR}}(\theta_{OP}))
+
\frac{\mu}{2}
\|\theta_{OP} - \theta_M^{\mathrm{BR}}(\theta_{OP})\|^2.
\]

Since
\[
\mathcal{DPR}(\theta_{OP},\theta_M^{\mathrm{BR}}(\theta_{OP}))
\ge
\min_{\theta_D,\theta_M}
\mathcal{DPR}(\theta_D,\theta_M),
\]
we conclude that
\[
\mathcal{PR}(\theta_{OP})
-
\min_{\theta_D,\theta_M}
\mathcal{DPR}(\theta_D,\theta_M)
\ge
\frac{\mu}{2}
\|\theta_{OP} - \theta_M^{\mathrm{BR}}(\theta_{OP})\|^2.
\]
\end{proof}

\begin{theorem}[Upper bound]
Assume that \(\ell(z;\theta)\) is \(L\)-Lipschitz in \(\theta\) for all \(z\). Then
\begin{equation}
    \min_{\theta}\mathcal{PR}(\theta) - \min_{\theta_D,\theta_M}\mathcal{DPR}(\theta_D,\theta_M)
    \le
    L \|\theta_D^* - \theta_M^*\|,
\end{equation}
where
\((\theta_D^*,\theta_M^*)
\in
\argmin_{\theta_D,\theta_M}\mathcal{DPR}(\theta_D,\theta_M)\)
and
\(\theta_{OP} \in \argmin_{\theta}\mathcal{PR}(\theta)\).
\end{theorem}

\begin{proof}
Since \(\ell(z;\theta)\) is \(L\)-Lipschitz in \(\theta\), we have for all \(z\),
\[
|\ell(z;\theta_D^*) - \ell(z;\theta_M^*)|
\le
L \|\theta_D^* - \theta_M^*\|.
\]
Taking expectation with respect to \(z \sim \mathcal{D}(\theta_D^*)\),
\[
\mathbb{E}_{z\sim\mathcal{D}(\theta_D^*)}
\big[\big|
\ell(z;\theta_D^*) - \ell(z;\theta_M^*)
\big|\big]
\le
L \|\theta_D^* - \theta_M^*\|.
\]
By Jensen's inequality (\(|\mathbb{E}[X]| \le \mathbb{E}[|X|]\)),
\begin{align*}
L \|\theta_D^* - \theta_M^*\|
&\ge
\left|
\mathbb{E}_{z\sim\mathcal{D}(\theta_D^*)}
[\ell(z;\theta_D^*) - \ell(z;\theta_M^*)]
\right| \\
&=
\left|
\mathcal{PR}(\theta_D^*)
-
\mathcal{DPR}(\theta_D^*,\theta_M^*)
\right|.
\end{align*}
Then, because $ \mathcal{PR}(\theta_D^*) \ge \mathcal{DPR}(\theta_D^*,\theta_M^*), $
the difference is nonnegative and we obtain
\[
L \|\theta_D^* - \theta_M^*\|
\ge
\mathcal{PR}(\theta_D^*)
-
\mathcal{DPR}(\theta_D^*,\theta_M^*).
\]
Finally, since \(\theta_{OP}\) minimizes \(\mathcal{PR}\),
\[
\mathcal{PR}(\theta_D^*)
\ge
\mathcal{PR}(\theta_{OP}).
\]
we have
\[
\min_{\theta}\mathcal{PR}(\theta) - \min_{\theta_D,\theta_M}\mathcal{DPR}(\theta_D,\theta_M)
\le
L \|\theta_D^* - \theta_M^*\|.
\]
\end{proof}

\subsection{Convexity of the decoupled performative risk}

Recall the following definitions: 

\begin{definition}[Jointly convex]
    A function $f: \mathbb{R}^d\times\mathbb{R}^d\to\mathbb{R}$ is \emph{jointly convex} if
    \begin{equation}
        f(\lambda x + (1-\lambda)x';\lambda y + (1-\lambda)y') \leq \lambda f(x;y) + (1-\lambda) f(x';y')~.
    \end{equation}
\end{definition}

\begin{definition}[Affinity]\label{def:affinity1}
    A function $f:\mathbb{R}^n\to\mathbb{R}^m$ is \emph{affine} if
    \begin{equation}
        f(\lambda_1x_1 + \ldots + \lambda_nx_n) = \lambda_1 f(x_1) + \ldots + \lambda_n f(x_n)
    \end{equation}
    where $\lambda_1 + \ldots + \lambda_n=1$.
\end{definition}

{
\setcounter{theorem}{8}
\renewcommand{\thetheorem}{\ref{th:convexity-dpr}}
\begin{theorem}
If $\ell(z;\theta_M)$ is jointly convex and $\mathcal{D}(\theta_D)$ follows the push-forward model where $\varphi(z_o;\theta_D)$ is affine in $\theta_D$, then $\mathcal{DPR}(\theta_D,\theta_M)$ is jointly convex.
\end{theorem}
}
\begin{proof}
Recall that
\begin{equation}
\mathcal{DPR}(\theta_D, \theta_M):=\mathbb{E}_{z\sim\mathcal{D}(\theta_0)}\big[\ell(z;\theta_M)\big]~.
\end{equation}
We wish to prove that
\begin{equation}
\mathcal{DPR}(\lambda\theta_D+(1-\lambda)\theta_D', \lambda\theta_M+(1-\lambda)\theta_M')\leq \lambda\mathcal{DPR}(\theta_D, \theta_M)+(1-\lambda)\mathcal{DPR}(\theta_D', \theta_M').
\end{equation}
Our starting point is
\begin{equation}
\mathcal{DPR}(\lambda\theta_D+(1-\lambda)\theta_D', \lambda\theta_M+(1-\lambda)\theta_M') = \mathbb{E}_{z\sim\mathcal{D}(\lambda \theta_D + (1-\lambda)\theta'_D)} \big[\ell(z;\lambda\theta_M+(1-\lambda)\theta_M')\big].
\end{equation}
In order to simplify the notation a bit, we introduce the following shorthand: $\tilde{\theta}_M:=\lambda\theta_M+(1-\lambda)\theta_M'$. The previous expression then becomes
\begin{equation}
\mathbb{E}_{z\sim\mathcal{D}(\lambda \theta_D + (1-\lambda)\theta'_D)} \big[\ell(z;\tilde{\theta}_M)\big].
\end{equation}
As the first step, we use the fact that $\mathcal{D}(\cdot)$ follows the push-forward model. We can rewrite the expectation as:
\begin{equation}
\mathbb{E}_{z\sim\mathcal{D}(\lambda \theta_D + (1-\lambda)\theta'_D)} \big[\ell(z;\tilde{\theta}_M)\big]=\mathbb{E}_{z_0\sim\mathcal{D}_0} \Big[\ell\big(\varphi(z_0; \lambda \theta_D + (1-\lambda)\theta'_D);\tilde{\theta}_M\big)\Big].
\end{equation}
Next, we use the fact that $\varphi$ is affine in $\theta_D$:
\begin{equation}
= \mathbb{E}_{z_0\sim\mathcal{D}_0} \Big[\ell\big(\lambda\varphi(z_0; \theta_D) + (1-\lambda)\varphi(z_0;\theta'_D);\tilde{\theta}_M\big)\Big].
\end{equation}
Now, notice that $\varphi(z_0;\theta_D)$ corresponds to $z\sim\mathcal{D}(\theta_D)$ and $\varphi(z_0;\theta_D')$ corresponds to $z'\sim\mathcal{D}(\theta_D')$. This means we can rewrite the expression like this
\begin{equation}
=\mathbb{E}_{z\sim\mathcal{D}(\theta_D),z'\sim\mathcal{D}(\theta_D')} \big[\ell(\lambda z + (1-\lambda)z';\tilde{\theta}_M)\big].
\end{equation}

We are now ready to replace $\tilde{\theta}_M$ by its original definition again. We show here the expression that we started with and what we have derived so far:
\begin{multline}
\mathbb{E}_{z\sim\mathcal{D}(\lambda \theta_D + (1-\lambda)\theta'_D)} \big[\ell(z;\lambda\theta_M+(1-\lambda)\theta_M')\big]\\
=\mathbb{E}_{z\sim\mathcal{D}(\theta_D),z'\sim\mathcal{D}(\theta_D')} \big[\ell(\lambda z + (1-\lambda)z';\lambda\theta_M+(1-\lambda)\theta_M')\big].
\end{multline}
As the final step, we use the fact that $\ell$ is jointly convex:
\begin{multline}
\mathbb{E}_{z\sim\mathcal{D}(\lambda \theta_D + (1-\lambda)\theta'_D)} \big[\ell(z;\lambda\theta_M+(1-\lambda)\theta_M')\big] \\
\leq \mathbb{E}_{z\sim\mathcal{D}(\theta_D), z'\sim\mathcal{D}(\theta'_D)} \big[\lambda \ell(z;\theta_M) + (1-\lambda) \ell(z';\theta'_M)\big].
\end{multline}
Notice now that if we split up the expectation on the right-hand side and move out the constant factors, we get this:
\begin{multline}
\mathbb{E}_{z\sim\mathcal{D}(\lambda \theta_D + (1-\lambda)\theta'_D)} \big[\ell(z;\lambda\theta_M+(1-\lambda)\theta_M')\big] \\
\leq\lambda \mathbb{E}_{z\sim\mathcal{D}(\theta_D)} \big[ \ell(z;\theta_M)\big] + (1-\lambda)\mathbb{E}_{ z'\sim\mathcal{D}(\theta'_D)} \big[ \ell(z';\theta'_M)\big]
\end{multline}
which is what we wanted to show.
\end{proof}

We provide another proof of the convexity of $\mathcal{DPR}(\theta_D,\theta_M)$ with a modification of the affinity condition. For this, we need to modify slightly the assumptions. 

\begin{definition}[$\gamma$-strongly jointly convex]
    A function $f: \mathbb{R}^d\times\mathbb{R}^d\to\mathbb{R}$ is $\gamma$-\emph{strongly jointly convex} if
    \begin{equation}
        f(x',y')-f(x,y)\geq \frac{\partial f(x,y)}{\partial x}(x'-x)+\frac{\partial f(x,y)}{\partial y}(y'-y) + \frac{\gamma}{2}\left(\|x'-x\|^2+\|y'-y\|^2\right)~.
    \end{equation}
\end{definition}

\begin{definition}[Affinity, alternative]
    A function $f:\mathbb{R}^n\to\mathbb{R}^m$ is \emph{affine} if
    \begin{equation}
        f(x')-f(x) = \mathbf{J}_f(x) (x'-x)
    \end{equation}
    where $\mathbf{J}_f$ is the Jacobian matrix of $f$, $\mathbf{J}_f=\left(\frac{\partial f}{\partial x_1}\cdots\frac{\partial f}{\partial x_n}\right)$. This definition is equivalent to definition~\ref{def:affinity1}.
\end{definition}

\begin{theorem}
If $\ell(z;\theta_M)$ is $\gamma$-strongly jointly convex and $\mathcal{D}(\theta_D)$ follows the push-forward model where $\varphi(z_o;\theta_D)$ is affine in $\theta_D$, such that its derivative in $\theta_D$ is constant in $\theta_D$ for a given $z_o$,  i.e., $\frac{\partial\varphi(z_o; \theta_D)}{\partial\theta_D}=g(z_o)$ for some $g(\cdot)$, then $\mathcal{DPR}(\theta_D,\theta_M)$ is $(\lambda\gamma)$-strongly jointly convex, where $\lambda=\min(1, \mathbb{E}_{z_o\sim\mathcal{D}_o}\left[\left\|g(z_o)\right\|^2\right])$ where $g(\cdot)$ is the partial derivative of $\varphi(z_o;\theta_D)$.
\end{theorem}

\begin{proof}
    We need to show that
    \begin{multline}
        \mathcal{DPR}(\theta_D',\theta_M') - \mathcal{DPR}(\theta_D,\theta_M) \geq \nabla_{\theta_M} \mathcal{DPR}(\theta_D,\theta_M)(\theta'_M - \theta_M)\\
        + \nabla_{\theta_D} \mathcal{DPR}(\theta_D,\theta_M)(\theta'_D - \theta_D)
        +\frac{\gamma\lambda}{2} (\|\theta'_M - \theta_M\|^2 + \|\theta'_D - \theta_D\|^2).
    \end{multline}
    We begin with $\mathcal{DPR}(\theta_D',\theta_M') - \mathcal{DPR}(\theta_D,\theta_M)$ and rewrite it in terms of the expectation:
    \begin{equation}
        \mathcal{DPR}(\theta_D',\theta_M') - \mathcal{DPR}(\theta_D,\theta_M) =
        \mathbb{E}_{z'\sim\mathcal{D}(\theta_D'), z\sim\mathcal{D}(\theta_D)}\big[\ell(z';\theta_M')-\ell(z;\theta_M)\big]~.
    \end{equation}
    Next, we use the fact that $\ell$ is $\gamma$-strongly jointly convex:
    \begin{multline}
        \mathcal{DPR}(\theta_D',\theta_M') - \mathcal{DPR}(\theta_D,\theta_M)\\
        \geq
        \mathbb{E}_{z'\sim\mathcal{D}(\theta_D'), z\sim\mathcal{D}(\theta_D)}\Big[\nabla_{\theta_M}\ell(z;\theta_M)(\theta_M'-\theta_M) + \nabla_z\ell(z;\theta_M)(z'-z)\\
        +\frac{\gamma}{2}\left(\|\theta_M'-\theta_M\|^2+\|z'-z\|^2\right)\Big]~.
    \end{multline}
    The second term within the expectation can be turned back into $\mathcal{DPR}(\theta_D;\theta_M)$ and terms independent of $z,z'$ can be moved out:
    \begin{multline}
        =\nabla_{\theta_M}\mathcal{DPR}(\theta_D,\theta_M)(\theta_M'-\theta_M)+\frac\gamma2\|\theta_M'-\theta_M\|^2\\
        +\mathbb{E}_{z'\sim\mathcal{D}(\theta_D'), z\sim\mathcal{D}(\theta_D)}\Big[\nabla_z\ell(z;\theta_M)(z'-z)+\frac\gamma2\|z'-z\|^2\Big]~.
    \end{multline}
    We now rewrite the expectation to be over the base distribution, $\mathcal{D}_o$:
    \begin{multline}
        =\nabla_{\theta_M}\mathcal{DPR}(\theta_D,\theta_M)(\theta_M'-\theta_M)+\frac\gamma2\|\theta_M'-\theta_M\|^2\\
        +\mathbb{E}_{z_o\sim\mathcal{D}_o}\bigg[\left.\frac{\partial\ell(z;\theta_M)}{\partial z}\right|_{z=\varphi(z_o;\theta_D)}\big(\varphi(z_o;\theta_D')-\varphi(z_o;\theta_D)\big)\\
        +\frac\gamma2\big\|\varphi(z_o;\theta_D')-\varphi(z_o;\theta_D)\big\|^2\bigg]~.
    \end{multline}
    The gradient of $\ell$ was rewritten for clarity.
    We can now use the fact that $\varphi(z_o;\theta_D)$ is affine in $\theta_D$:
    \begin{multline}
        =\nabla_{\theta_M}\mathcal{DPR}(\theta_D,\theta_M)(\theta_M'-\theta_M)+\frac\gamma2\|\theta_M'-\theta_M\|^2\\
        +\mathbb{E}_{z_o\sim\mathcal{D}_o}\Bigg[\left.\frac{\partial\ell(z;\theta_M)}{\partial z}\right|_{z=\varphi(z_o;\theta_D)}\frac{\partial\varphi(z_o;\theta_D)}{\partial\theta_D}(\theta_D'-\theta_D)\\
        +\frac\gamma2\left\|\frac{\partial\varphi(z_o;\theta_D)}{\partial\theta_D}\right\|^2\|\theta_D'-\theta_D\|^2\Bigg]~.
    \end{multline}
    We can see now that the first term in the expectation corresponds to the chain rule of derivatives. In the second term, we replace the derivative of $\varphi$ with the function $g(z_o)$ which we can do because the derivative of an affine function is constant:
    \begin{multline}
        =\nabla_{\theta_M}\mathcal{DPR}(\theta_D,\theta_M)(\theta_M'-\theta_M)+\frac\gamma2\|\theta_M'-\theta_M\|^2\\
        +\mathbb{E}_{z_o\sim\mathcal{D}_o}\Bigg[\nabla_{\theta_D}\Big(\ell\big(\varphi(z_o;\theta_D);\theta_M\big)\Big)(\theta_D'-\theta_D)
        +\frac\gamma2\left\|g(z_o)\right\|^2\|\theta_D'-\theta_D\|^2\Bigg]~.
    \end{multline}
    We can recognize another $\mathcal{DPR}(\theta_M;\theta_D)$ term in there and we move everything out of the expectation that is independent of it:
    \begin{multline}
        =\nabla_{\theta_M}\mathcal{DPR}(\theta_D,\theta_M)(\theta_M'-\theta_M)+\frac\gamma2\|\theta_M'-\theta_M\|^2\\
        +\nabla_{\theta_D}\mathcal{DPR}(\theta_D,\theta_M)(\theta_D'-\theta_D)
        +\frac\gamma2\mathbb{E}_{z_o\sim\mathcal{D}_o}\left[\left\|g(z_o)\right\|^2\right]\|\theta_D'-\theta_D\|^2~.
    \end{multline}
    Now, per the assumption, $\lambda$ is $\leq 1$ and $\leq \mathbb{E}_{z_o\sim\mathcal{D}_o}\left[\left\|g(z_o)\right\|^2\right]$, so, we can introduce $\lambda$ as follows:
    \begin{multline}
        \mathcal{DPR}(\theta_D',\theta_M') - \mathcal{DPR}(\theta_D,\theta_M)\\
        \geq\nabla_{\theta_M}\mathcal{DPR}(\theta_D,\theta_M)(\theta_M'-\theta_M)+\frac{\lambda\gamma}2\|\theta_M'-\theta_M\|^2
        +\nabla_{\theta_D}\mathcal{DPR}(\theta_D,\theta_M)(\theta_D'-\theta_D)
        +\frac{\lambda\gamma}2\|\theta_D'-\theta_D\|^2,
    \end{multline}
    which is what we wanted to show.
\end{proof}

\subsection{Effect of deception cost on risk gap}\label{appendix:deception-cost-risk}
\begin{theorem}[Gap between performative and decoupled optima with deception cost]
Assume that \(\mathcal{L}(\hat y, y)\) is \(L\)-Lipschitz in \(\hat y\) for all \(y\).
Then, under the setting of perceived deception cost with $\mathbb{E}_{(x,y)\sim D(\theta_D)}[|f_{\theta_D}(x)-f_{\theta_M}(x)|] \le c'$, the gap in risk between the performative optimum and the decoupled optimum is bounded by
\begin{equation}
    \mathcal{PR}(\theta_{OP})
    -
    \mathcal{DPR}(\theta_D^*,\theta_M^*)
    \;\le\;c'L.
\end{equation}
\end{theorem}
\begin{proof}
Since \(\mathcal{L}(\hat y, y)\) is \(L\)-Lipschitz in \(\hat y\), we have for all \(y\),
\begin{equation}
|\mathcal{L}(\hat{y}_1, y) - \mathcal{L}(\hat{y}_2, y)|
\le
L \|\hat{y}_1 - \hat{y}_2\|.
\end{equation}
Taking expectation with respect to \((x,y) \sim \mathcal{D}(\theta_D^*)\)
and setting $\hat{y}_1=f_{\theta_D^*}(x)$ and $\hat{y}_2=f_{\theta_M^*}(x)$,
where $(\theta_D^*, \theta_M^*)$ is the minimum of Eq~\eqref{eq:optimization-problem-perceived-deception}, we get
\begin{equation}
\mathbb{E}_{(x,y)\sim\mathcal{D}(\theta_D^*)}
\big[\big|
\mathcal{L}(f_{\theta_D^*}(x), y) - \mathcal{L}(f_{\theta_M^*}(x), y)
\big|\big]
\le
L |f_{\theta_D^*}(x) - f_{\theta_M^*}(x)|.
\label{eq:lipschitz-proof-appendix}
\end{equation}
Combining Eq~\eqref{eq:lipschitz-proof-appendix} with the constraint from Eq~\eqref{eq:optimization-problem-perceived-deception}, we get
    \begin{align}
Lc' &\geq\underset{(x, y)\sim \mathcal{D}(\theta_D^*)}{\mathbb{E}}
\Bigl[L\bigl|f_{\theta_D^*}(x)-f_{\theta_M^*}(x)\bigr|\Bigr]\nonumber\\
&\geq \underset{(x, y)\sim \mathcal{D}(\theta_D)}{\mathbb{E}}\Bigl[\big|\mathcal{L}\big(f_{\theta_D^*}(x), y\big) -\mathcal{L}\big(f_{\theta_M^*}(x), y\big)\big|\Bigr].
\end{align}

By Jensen's inequality (\(|\mathbb{E}[X]| \le \mathbb{E}[|X|]\)),
\begin{align*}
L c'
&\ge
\left|
\underset{(x, y)\sim \mathcal{D}(\theta_D^*)}{\mathbb{E}}\Bigl[\mathcal{L}\big(f_{\theta_D^*}(x), y\big) -\mathcal{L}\big(f_{\theta_M^*}(x), y\big)\Bigr]
\right|\\
&=
\left|
\mathcal{PR}(\theta_D^*)
-
\mathcal{DPR}(\theta_D^*,\theta_M^*)
\right|.
\end{align*}
Then, because $ \mathcal{PR}(\theta_D^*) \ge \mathcal{DPR}(\theta_D^*,\theta_M^*), $
the difference is nonnegative and we obtain
\[
L c'
\ge
\mathcal{PR}(\theta_D^*)
-
\mathcal{DPR}(\theta_D^*,\theta_M^*).
\]
Finally, since \(\theta_{OP}\) minimizes \(\mathcal{PR}\), then
\[
\mathcal{PR}(\theta_D^*)
\ge
\mathcal{PR}(\theta_{OP}),
\]
and we have
\[
\mathcal{PR}(\theta_{OP}) - \mathcal{DPR}(\theta_D^*,\theta_M^*)
\le
L c'.
\]
\end{proof}

\subsection{Lipschitzness of cross-entropy with logits}\label{appendix:ce-lipschitz}
\begin{proposition}
For $x \in \mathbb{R}^C$ and $y \in \{0,1,\dots,C-1\}$, define
\[
\ell(x,y) = -x_y + \log\!\left(\sum_{k=0}^{C-1} e^{x_k}\right).
\]
Then, for every fixed $y$, the map $x \mapsto \ell(x,y)$ is globally $\sqrt{2}$-Lipschitz with respect to the Euclidean norm.
\end{proposition}

\begin{proof}
Let
\[
p = \operatorname{softmax}(x), \qquad
p_j = \frac{e^{x_j}}{\sum_{k=0}^{C-1} e^{x_k}}
\quad (j=0,\dots,C-1),
\]
and let $e_y \in \mathbb{R}^C$ denote the $y$-th standard basis vector, i.e., the one-hot vector where only the $y$-th entry is set to 1. Then differentiating
\[
\ell(x,y) = -x_y + \log\!\left(\sum_{k=0}^{C-1} e^{x_k}\right),
\]
with respect to $x$ gives
\[
\nabla_x \ell(x,y) = - e_y + p.
\]

Hence
\[
\|\nabla_x \ell(x,y)\|_2^2
= \|p - e_y\|_2^2
= \|p\|_2^2 + \|e_y\|_2^2 - 2\langle p, e_y\rangle
= \|p\|_2^2 + 1 - 2p_y.
\]
Now $p$ is a probability vector, so $0\leq p_j \leq 1$ for all $j$ and $\sum_j p_j = 1$. Therefore
\[
\|p\|_2^2 = \sum_{j=0}^{C-1} p_j^2 \le \sum_{j=0}^{C-1} p_j = 1.
\]
It follows that
\begin{equation}
\|\nabla_x \ell(x,y)\|_2^2
\le 1 + 1 - 2p_y
\le 2.
\end{equation}
Thus
\begin{equation}
\|\nabla_x \ell(x,y)\|_2 \le \sqrt{2}
\qquad \text{for all } x \in \mathbb{R}^C.
\end{equation}
Since a differentiable function whose gradient norm is bounded by $L$ is $L$-Lipschitz, we conclude that
\[
|\ell(x,y) - \ell(x',y)|
\le \sqrt{2}\,\|x-x'\|_2
\qquad \text{for all } x,x' \in \mathbb{R}^C.
\]
Therefore $\ell(\cdot,y)$ is globally $\sqrt{2}$-Lipschitz.
\end{proof}

\section{Details of the Example~\ref{ex:biased-coin} and~\ref{ex:biased-coin-dpr}}
\label{appendix:details-example}
\subsection{Performative Risk case (Example~\ref{ex:biased-coin}).}

The same details can be consulted in \cite{peformativeprediction2020perdomo}. We repeat them in this manuscript, so it is easier to compare them with the decoupled risk case, which we introduce in this work.

We want to learn a model $f_{\theta}(x) = \theta x + \frac{1}{2}$ with the squared error as the loss $\ell(z,\theta)=(z-f_{\theta}(x))^2$. The distribution map is modeled though the conditional probability of $Y|X \sim \text{Bernuilli}(\frac{1}{2} + \mu X+ \varepsilon \theta X)$ with support in $\{\pm 1\}$. We consider $\varepsilon>0$, $\theta\in[0,1]$ and $\mu\in(0,\frac{1}{2})$.

Let's first expand the performative risk by using the law of total expectation. 

\begin{equation*}
    \mathcal{PR}(\theta) = \mathbb{E}_{(x,y)\sim \mathcal{D}(\theta)}[(Y-f_\theta(X))^2] = \mathbb{E}_{X}\mathbb{E}[(Y-f_\theta(X))^2 | X].
\end{equation*}

Expanding the square, we have

\begin{align*}
     \mathbb{E}[(Y-f_\theta(X))^2 | X] &= \mathbb{E}[Y^2 + 2Yf_\theta(X) + f_\theta(X)^2|X] \\ &=\mathbb{E}[Y^2|X] + 2 \mathbb{E}[Y|X] f_{\theta}(X) + f_{\theta}(X)^2.
\end{align*}

Adding and subtracting the term $\mathbb{E}^2[Y|X]$, we can write

\begin{align}
\label{eq:expansion-pr-example}
     \mathbb{E}[(Y-f_\theta(X))^2 | X] &=\mathbb{E}[Y^2|X] + 2 \mathbb{E}[Y|X] f_{\theta}(X) + f_{\theta}(X)^2 + \mathbb{E}^2[Y|X] - \mathbb{E}^2[Y|X] \nonumber\\ &= \text{Var}(Y|X) + (\mathbb{E}[Y|X] - f_{\theta}(X))^2.
\end{align}

Finally, plugging the variance and mean of the Bernuilli distribution, we get

\begin{equation}
\label{eq:perf-risk-example}
    PR(\theta) = \mathbb{E}_X[\frac{1}{4}-2\mu\theta X^2 + (1-2\varepsilon)\theta^2X^2] = \frac{1}{4}-2\mu\theta + (1-2\varepsilon)\theta^2.
\end{equation}

Note that because $X$ is supported in $\{\pm1\}$, $X^2=1$.

\paragraph{Stable point}

Recall that the definition of the stable point considers the distribution to be fixed (and given by that same stable point).

\begin{equation*}
    \theta_{ST} = \argmin_{\theta \in \Theta} \mathbb{E}_{z\sim\mathcal{D}(\theta_{ST})}[\ell(z;\theta)].
\end{equation*}

We need to set the gradient of the expression to zero, which we want to find the minimum. 

\begin{align*}
\nabla_{\theta} \mathbb{E}_{z\sim\mathcal{D}(\theta_{ST})}[\ell(z;\theta)] &= \mathbb{E}_{z\sim\mathcal{D}(\theta_{ST})}[\nabla_{\theta}\ell(z;\theta)] \\ &= \mathbb{E}_{z\sim\mathcal{D}(\theta_{ST})}[\nabla_{\theta}(Y-f_{\theta}(X))^2]  \\ &= \mathbb{E}_{z\sim\mathcal{D}(\theta_{ST})}[2(Y-f_{\theta}(X))X] = 0.
\end{align*}

We are looking for $\theta_{ST}$ and from the law of total expectation, we know that

\begin{align*}
\mathbb{E}_{z\sim\mathcal{D}(\theta_{ST})}[2(Y-f_{\theta_{ST}}(X))X] &= \mathbb{E}_X\mathbb{E}[2(Y-f_{\theta_{ST}}(X))X|X] \\ &= \mathbb{E}_X\mathbb{E}[2XY - 2Xf_{\theta_{ST}}(X)|X] \\ &= \mathbb{E}_X[2\mu X^2 + 2\varepsilon\theta_{ST} X^2 - 2\theta_{ST}X^2] \\ &= 2(\mu + (\varepsilon -1)\theta_{ST}) = 0.
\end{align*}

Again, we used that $X^2 = 1$. Solving the equation, we find

\begin{equation*}
    \theta_{ST} = \frac{\mu}{1-\varepsilon},
\end{equation*}

and the corresponding risk is

\begin{equation*}
    \mathcal{PR}(\theta_{ST}) = \frac{1}{4} - \frac{\mu^2}{(1-\varepsilon)^2}.
\end{equation*}

\paragraph{Optimal point}

To calculate the optimal point, we need to calculate the gradient of the $\mathcal{PR}(\theta)$ with respect to $\theta$ and set it equal to zero. From equation \ref{eq:perf-risk-example}, we have:

\begin{equation*}
    \nabla \mathcal{PR(\theta)} = \nabla (\frac{1}{4}-2\mu\theta + (1-2\varepsilon)\theta^2) = 2 ( (1-2\varepsilon) \theta -\mu).
\end{equation*}

Then, it is easy to see that
\begin{equation*}
    \nabla \mathcal{PR(\theta)}\big|_{\theta=\theta_{OP}} = 0 \iff \theta_{OP} = \frac{\mu}{1-2\varepsilon}.
\end{equation*}

As $\theta\in[0,1]$, for $\theta_{OP}$ to exist, we require $\frac{\mu}{1-2\varepsilon} < 1$. The associated risk of the optimal point is then

\begin{equation*}
    \mathcal{PR}(\theta) = \frac{1}{4} - \frac{\mu^2}{1-2\varepsilon}.
\end{equation*}

\subsection{Decoupled Risk case (Example~\ref{ex:biased-coin-dpr}).} 

Now, we have $f_{\theta_D}(x) = \theta_D x + \frac{1}{2}$ and $Y|X \sim \text{Bernuilli}(\frac{1}{2} + \mu X+ \varepsilon \theta_D X)$ with support also in $\{\pm 1\}$. Plugging the corresponding decoupled terms in eq. \ref{eq:expansion-pr-example}, in this case, we get

\begin{equation*}
    \mathcal{DPR}(\theta_D,\theta_M) = \frac{1}{4} + \theta_M^2 - 2\varepsilon\theta_D\theta_M - 2 \mu\theta_M.
\end{equation*}

We have found that $\mathcal{DPR}(\theta_D, \theta_M)$ is monotonically decreasing in $\theta_D$ because $\varepsilon>0$. Therefore, the minimum would be reached in the upper bound of $\theta_D \in [0,1]$, i.e $\theta_D=1$. Thus, the problem gets simplified to

\begin{equation*}
    \argmin_{\theta_M\in\Theta} \mathcal{DPR}(1, \theta_M) = \frac{1}{4} + \theta_M^2 - 2\varepsilon\theta_M - 2\mu\theta_M \iff \theta_M = \mu + \varepsilon.
\end{equation*}

Therefore, the decoupled optimum is $(\theta_D^*,\theta_M^*) = (1,\mu+\varepsilon)$ and its associated risk

\begin{equation*}
    \mathcal{DPR}(\theta_D^*,\theta_M^*) = \frac{1}{4} - (\mu+\varepsilon)^2.
\end{equation*}

\section{Experimental details}
\label{appendix:experimental-details}

We repeat all experiments for 5 runs and plot the mean of the metrics. We add the standard deviation as a shape around the mean to account for the variance of the runs. 

\subsection{Pricing}
\label{appendix:pricing-details}
We initialize $\theta_0 \sim\mathcal{N}(6 \cdot \mathbf{1}; 1)$, $z_0 \sim\mathcal{N}(9 \cdot \mathbf{1}; 1)$, $\theta\sim\mathcal{N}(\mathbf{0},I_d)$ with $\mathbf{1} = (1,1,\ldots,1) \in \mathbb{R}^d$. We set the performative effect $\varepsilon=1.5$. We sample $n=1000$ inputs. The learning rate is fixed at $\eta=0.1$. We consider $\theta\in[-5,5]$

This setting has a closed-form solution for the stable point $\theta_{ST} = \frac{\mu_0}{\varepsilon}$ and for the optimal point $\theta_{OP} = \frac{\mu_0}{2\varepsilon}$. 

\paragraph{Calculating the stable point.}
Following the definition of the stable point, 

\begin{align*}
    \theta_{ST} = \argmin_{\theta\in\Theta}\mathbb{E}_{z\sim\mathcal{D}(\theta_{ST})}[\ell(z;\theta)].
\end{align*}

The distribution is considered to be fixed, to calculate the minimum, we need to calculate the gradient and set it equal to zero.

\begin{align*}
    \nabla_{\theta} \mathbb{E}_{z\sim\mathcal{D}(\theta_{ST})}[\ell(z;\theta)] &= \mathbb{E}_{z\sim\mathcal{D}(\theta_{ST})}[\nabla_{\theta}\ell(z;\theta)]\\& = \mathbb{E}_{z\sim\mathcal{D}(\theta_{ST})}[-z] \\&= - (z_0 - \varepsilon\theta_ST) = 0.
\end{align*}

Solving the equation, we have 

\begin{equation*}
    \theta_{ST} = \frac{z_0}{\varepsilon}.
\end{equation*}

\paragraph{Calculating the optimal point.}
The performative risk is

\begin{align*}
    \mathcal{PR}(\theta) &= \mathbb{E}_{z\sim\mathcal{N}(z_0 - \varepsilon\theta, I_d)}[-(\theta_0 + \theta)^Tz] \\
    &=-(\theta_0 + \theta)^T(z_0 - \varepsilon\theta) = \varepsilon\theta^2 + (\varepsilon\theta_0-z_0)\theta - z_0.
\end{align*}

The minimum of this value, i.e., the optimal point, is

\begin{equation}
    \theta_{OP} = \frac{z_0-\varepsilon\theta_0}{2\varepsilon}.
\end{equation}

\subsection{Give me some credit}

Following \citet{peformativeprediction2020perdomo}, we balance the dataset and normalize each column of the \textit{GiveMeSomeCredit} dataset to have zero mean and unit standard deviation. However, we do not set only some features to be the "performative" ones, i.e., all columns of $x_o$ can be changed when calculating $x=x_o + \varepsilon \nabla_x f_\theta(x)$.

We set $\varepsilon=10$ to get a significant performative effect. If $\varepsilon$ is small, $x_o\approx x$ and the performative effect disappears. We use $n=120{,}000$ samples and a learning rate of $\mu=0.1$. We use a 2-layer Multilayer Perceptron with a $100$ neurons as a hidden layer and a \texttt{ReLU} as a nonlinearity between layers. We initialize each weight with a normal distribution $\theta \sim \mathcal{N}(0,1)$. 

This experiment does not have a closed-form solution. 

\section{Decoupled stable points}
\label{appendix:decoupled-stable}
A decoupled stable point is the minimum of a distribution induced by any other model $\theta_D$

\begin{definition}[Decoupled stable point] A decoupled stable point is defined as
\begin{equation}
    \theta_M^{*} = \argmin_{\theta_M\in\Theta} \mathcal{DPR}(\theta_D,\theta_M).
\end{equation}
\end{definition}

Therefore, it follows

\begin{proposition}
\label{th:001}
If $\theta_{ST}$ is a stable point of the performative risk, then $\nabla_{\theta_M} \mathcal{DPR} (\theta_{ST},\theta_{ST})= 0$ 
\end{proposition}

With the correct assumptions, there is one stable point for each $\theta_D$.
Note that a \textit{standard} stable point can be formulated as a point that lies in the $\theta_M = \theta_D$ plane where the partial derivative of \(\mathcal{DPR}\) w.r.t.\ \(\theta_M\) vanishes.

\subsection{Proof of proposition \ref{th:001}}
Let us fix a particular \(\theta_D\) and thus a particular data distribution \(\mathcal{D}(\theta_D)\).
Let \(\theta_M^*\) be the optimal model parameter vector on that distribution.
We can conceptualize it as a function of a \(\theta_D\) and express it in terms of the decoupled risk:
\begin{align*}
    \theta^{*}_M(\theta_D) := \,&\argmin_{\theta_M \in \Theta} \mathcal{DPR}(\theta_D, \theta_{M}) \\
=\, &\argmin_{\theta_M \in \Theta} \mathbb{E}_{z \sim \mathcal{D}(\theta_{D})}  \big[\ell(z, \theta_M)\big].
\end{align*}%
\begin{definition}[Stable point]
The stable point is defined as the \emph{fixed point} of the above function:
\begin{align*}
    \theta_{ST} = \theta_M^*(\theta_{ST}) ~.
\end{align*}
\end{definition}
It follows that
\begin{align*}
    \left.\frac{\partial \mathcal{DPR}(\theta_{D}, \theta_{M})}{\partial \theta_M}\right|_{\theta_M=\theta_{ST}, \theta_D=\theta_{ST}}  = 0.
\end{align*}
because once we have both $\theta_D$ and $\theta_M$ set to $\theta_{ST}$,
we cannot make $\mathcal{DPR}$ smaller by varying $\theta_M$.

\end{document}